\documentclass{article}

\usepackage{mathtools}
\usepackage{authblk}
\usepackage{times}
\usepackage{microtype}
\usepackage{fullpage}

\usepackage[utf8]{inputenc} 
\usepackage[T1]{fontenc}    
\usepackage{hyperref}       
\usepackage{url}            
\usepackage{booktabs}       
\usepackage{amsfonts}       
\usepackage{nicefrac}       
\usepackage{microtype}      
\usepackage{subfig}
\usepackage{verbatim} 
\usepackage{mathtools}
\usepackage{xcolor}
\usepackage{bm}
\usepackage[ruled]{algorithm2e}
\usepackage[]{algpseudocode}
\algdef{SE}[DOWHILE]{Do}{doWhile}{\algorithmicdo}[1]{\algorithmicwhile\ #1}%
\usepackage{booktabs}
\usepackage[utf8]{inputenc}
\usepackage{Definitions}
\usepackage{paralist, tabularx}
\usepackage[numbers,square]{natbib}
\usepackage[normalem]{ulem}
\usepackage{cancel}
\usepackage{pifont}
\usepackage{stackengine}

\newcommand{\bnm}[1]{\left\| #1 \right\|}

\newcommand{\dps}{\Delta \bm{\psi}^*}

\newcommand{\dbb}{\Delta b^*}
\newcommand{\cmark}{\text{\ding{51}}}%
\newcommand{\xmark}{\text{\ding{55}}}%

\newcommand{\nn}{\nonumber}
\newcommand{\xhdr}[1]{\vspace{1.5mm}\noindent{{\bf #1.}}}

\usepackage{wrapfig}
\newcommand{\explain}[2]{\underbrace{#1}_{#2}}
\newcommand{\explainup}[2]{\overbrace{#1}^{#2}}
\newcommand{\minz}[1]{\underset{#1}{\text{minimize}}}

\newcommand{\df}{\Delta \wb^*}
\newcommand{\inp}[2]{{#2}^\top {#1}}
\newcommand{\wbvsl}{\wb^*(\Vcal\cp(\Scal\cup \Lcal))}
\newcommand{\wbvs}{\wb^*(\Vcal\cp\Scal)}
\newcommand{\bbvs}{b^*(\Vcal\cp\Scal)}

\newcommand{\prvsl}{\pred^*(\Vcal\cp(\Scal\cup \Lcal))}
\newcommand{\prvs}{\pred^*(\Vcal\cp\Scal)}

\newcommand{\oursub}{\xhdr}
\newcommand{\vcs}{\Vcal\cp\Scal}
\newcommand{\featdim}{m}
 
\usepackage{graphicx}
\newcommand\blfootnote[1]{%
  \begingroup
  \renewcommand\thefootnote{}\footnote{#1}%
  \addtocounter{footnote}{-1}%
  \endgroup
}
\newcommand{\tf}[1]{\Phi(#1)}

\title{Classification Under Human Assistance}

\author[1]{Abir De$^*$}
\author[2]{Nastaran Okati$^*$}
\author[2]{Ali Zarezade}
\author[2]{Manuel Gomez-Rodriguez}
\affil[1]{%
  {IIT Bombay, abir@cse.iitb.ac.in}
}
\affil[2]{%
  {MPI for Software Systems, \{nastaran, zarezade, manuelgr\}@mpi-sws.org}
}
\date{}

\begin{document}
\maketitle
\blfootnote{$^{*}$ Equal contributions.} 

\begin{abstract}
Most supervised learning mo\-dels are trained for full automation. However, their predictions are sometimes worse than 
those by human experts on some specific instances.
Motivated by this empirical observation, our goal is to design classifiers that are optimized to operate under different 
automation levels.
More specifically, we focus on convex margin-based classifiers and first show that the problem is NP-hard.
Then, we further show that, for  support vector machines, the corres\-pon\-ding objective function can be expressed as the 
difference of two functions $f = g - c$, where $g$ is monotone, non-negative and $\gamma$-weakly submodular, and $c$ 
is non-negative and mo\-du\-lar.
%
This representation allows us to utilize a recently introduced deterministic greedy algorithm, as well as a more efficient randomized 
variant of the algorithm, which en\-joy approximation guarantees at solving the problem.
Ex\-pe\-ri\-ments on synthetic and real-world data from several applications in medical diag\-no\-sis illustrate our theoretical 
findings and demonstrate that, under human assistance, supervised learning models trained to operate under different 
automation levels can outperform those trained for full automation as well as humans operating alone.

\end{abstract}

\section{Introduction} 
\label{sec:introduction}
In recent years, machine learning models have matched, or even surpassed, the average performance of human experts at tasks for which 
intelligence is required~\cite{hinton2013icassp,hinton2012nips,silver2016mastering,sutskever2014nips}.
As a consequence, there is a widespread discussion on the possibility of letting machine learning models take high-stake decisions---the 
promise is that the timeliness and quality of the decisions would greatly improve.
For example, in medical diagnosis, patients would not need to wait for months to be diagnosed by a specialist. 
In content moderation, online publishers could moderate toxic comments before they trigger incivility in their platforms.
In software development, developers would easily find bugs in large software projects and would not need to spend long hours in code reviews.

Unfortunately, the decisions taken by machine learning models are still worse than those by human experts on some instances, where they 
make far more errors than average~\cite{raghu2019algorithmic}.
Motivated by this observation, there has been a paucity of work on developing machine learning models that are optimized to operate 
under different automation levels~\cite{de2020aaai, bordt2020humans, sontag2020, raghu2019algorithmic, wilder2020learning}---mo\-dels that are optimized to take decisions 
for a given fraction of the instances and leave the remaining ones to humans.
%
However, most of this work has developed heuristic algorithms that do not enjoy theoretical guarantees. One of the only exceptions is the work by 
~\citet{de2020aaai}, which has reduced the problem of ridge regression under different automation levels to the maximization of an 
$\alpha$-submodular function~\cite{gatmiry2019}.
In our work, rather than (ridge) regression, we focus on classification under human assistance and show that, for support vector machines, 
the problem can be solved using algorithms with theoretical guarantees.

More specifically, we first show that, for convex mar\-gin-based classi\-fiers, the problem of classification under human assistance is NP-hard. This is 
due to its combinatorial nature---for each potential meta-decision about which instances the classifier will decide upon, there is an optimal set of parameters 
for the classifier, however, the meta-decision is also something we seek to optimize.
Then, for support vector machines, we derive an alternative representation of the objective function as a difference of two functions $f = g - c$, 
where $g$ is monotone, non-negative, and $\gamma$-weakly submodular~\cite{bian2017guarantees, das2018approximate} and $c$ is non-negative 
and modular. 
Moreover, we further show that, in our problem, the submodularity ratio $\gamma$, which characterizes how close is the function $g$ to 
being submodular, can be lower bounded.
These properties allow a recently introduced deterministic greedy algorithm (Algorithm~\ref{alg:greedy}) as well as a 
more efficient randomized variant of the algorithm~\cite{harshaw2019submodular} to enjoy nontrivial approximation guarantees.
%
%

%
%
%
Finally, we experiment with synthetic and real-world data from several appli\-ca\-tions in medical diagnosis.
Our experiments on synthetic data reveal that, by outsourcing samples to humans during training, the resulting support vector machine is able to 
reduce the number of training samples inside or on the wrong side of the margin, among those samples it needs to decide upon.
Our experiments on real data demonstrate that, under human assistance, support vector machines trained to operate under different automation 
levels outperform those trained for full automation as well as humans ope\-ra\-ting alone\footnote{\scriptsize Our code and data are available in \url{https://github.com/Networks-Learning/classification-under-assistance}}. 
%
%

\oursub{Further related work}
There is a rapidly increasing line of work devoted to designing classifiers that are able to defer decisions~\cite{bartlett2008classification,cortes2016learning,el2010foundations,geifman2019selectivenet,geifman2018bias,hsu2020generalized, liu2019deep,ramaswamy2018consistent,thulasidasan2019combating, wiener2011agnostic, ziyin2020learning}. 
However, this line of work does not consider there is a human decision maker, with a human error model, who takes 
a decision whenever the classifiers defer it and the classifiers are trained to predict the labels of all samples in the 
training set, as in full automation.

Our work also relates to the area of active learning~\cite{chen2017active,cohn1995active,guo2008discriminative,hashemi2019submodular,hoi2006batch,sabato2014active,sugiyama2006active,willett2006faster},
where the goal is to determine which subset of training samples one should label so that a supervised machine learning model, trained on these samples, 
ge\-ne\-ra\-lizes well across the entire feature space during test.
However, there is a fundamental difference between our work and active learning. In our work, the trained model only needs to accurately predict samples 
which are close to the samples assigned to the machine during training time. In contrast, in active learning, the trained model needs to predict well \emph{any}
sample during test time.

Finally, our work advances the state of the art on human-machine collaboration~\citep{ghosh2020towards,grover2018learning,hadfield2016cooperative,haug2018teaching,kamalaruban2019interactive,macindoe2012pomcop,meresht2020learning,nikolaidis2017mathematical,nikolaidis2015efficient,radanovic2019learning,tschiatschek2019learner,wilson2018collaborative}.
However, rather than considering a setting in which the machine and the human interact with each other as most previous work, we develop algorithms that
learn to distribute decisions between humans and machines.

\section{Problem Formulation} 
\label{sec:formulation}
In this section, we formally introduce the problem of designing convex margin-based classifiers that are optimized to 
operate under different automation levels. Then, we show that, for convex margin-based classifiers, the problem is NP-hard. 
For simplicity, we will consider binary classification, however, our ideas can be extended to m-ary classification.

In binary classification, one needs to find a mapping function $f(\xb)$ between feature vectors $\xb \in \RR^{\featdim}$, with $\xb \sim p(\xb)$, 
and class labels $y \in \{-1, 1\}$, with $y \sim p(y \,|\, \xb)$.
To this end, one utilizes a training set $\Dcal=\{ (\xb_i, y_i) \}_{i \in \Vcal}$ to construct a mapping that works \emph{well} 
on an \emph{unseen} test set.
For margin-based classifiers, finding this mapping usually reduces to building a decision boundary defined by a set of 
parameters $\theta$ that separates feature vectors in the training set accor\-ding to their class labels.
One typically looks for the decision boundary that achieves the greatest classification accuracy in a test set by minimizing a convex 
loss function $\ell(h_{\theta}(\xb), y)$ over a training set, \ie, $\theta^{*} = \argmin_{\theta} \sum_{i \in \Vcal} \ell(h_{\theta}(\xb_i), y_i)$,
where $h_{\theta}(\xb_i)$ denotes the signed distance from the feature vector $\xb$ to the decision boundary.
Then, given an \emph{unseen} feature vector $\xb$ from the test set, the classifier predicts $f(\xb) = 1$ if $h_{\theta^{*}}(\xb_i) \geq 0$
and $f(\xb) = -1$ otherwise.

%
In binary classification under human assistance, for eve\-ry feature vector $\xb \in \RR^{\featdim}$, the mapping function $f(\xb)$ can resort to either a classifier or 
a human expert. 
For margin-based classifiers, finding the mapping then reduces to pi\-cking the subset of training samples $\Scal \subseteq \Vcal$ that are outsourced to
human experts, with $| \Scal | \leq n$, and building a decision boundary that separates feature vectors in the subset of training samples $\Scal^{c} = \Vcal \backslash \Scal$ 
according to their class labels. 
Using the same convex loss function $\ell(h_{\theta}(\xb), y)$ as in the standard binary classification, our goal is then to solve the following minimization
problem:
\begin{equation}  \label{eq:optimization-problem}
\begin{split}
\underset{\Scal, \theta}{\text{minimize}} & \quad \sum_{i \in \Vcal \backslash \Scal}  \ell(h_{\theta}(\xb_i), y_i) + \sum_{i \in \Scal} c(\xb_i, y_i)  \\
\text{subject to} &\quad |\Scal| \le n,
\end{split}
\end{equation} 
where $c(\xb, y)$ denotes the human \emph{error per sample}, which we will define more precisely in the next section. Here, we assume that human annotations are independent, which allows\- us to cast the total human error as the sum of human error per sample over all instances.
%
Moreover, we use a linear constraint on the number of examples outsourced to humans  because, in most practical scenarios, humans get paid every 
time they make a prediction---if they make $n$ predictions, they get paid $n$ times.

%

Given the optimal set $\Scal^{*}$, we can find the optimal pa\-ra\-me\-ter $\theta^{*} = \theta^{*}(\Vcal \backslash \Scal^{*})$ in polynomial time since, by assump\-tion, the loss $\ell(h_{\theta}(\xb_i), y_i)$ is convex.
Unfortunately, the follo\-wing Theorem tells us that, in general, we cannot expect to find both $\Scal^{*}$ and $\theta^{*}$ in polynomial time (proven in Appendix~\ref{app:np-hard}):
%
\begin{theorem}\label{thm:np-hard}
The problem of designing margin-based classifiers under human assistance defined in Eq.~\ref{eq:optimization-problem} is NP-Hard.
\end{theorem}
Moreover, given the solution to the above minimization problem, we would still need to decide whether to outsource an \emph{unseen} feature vector 
$\xb$ from the test set to a human expert even if $\xb \neq \xb_i$ for all $i \in \Vcal$.
To this end, we could train an additional model $\pi(d \,|\, \xb)$ to decide which samples to outsource to a human using the labeled set 
$\{ (\xb_i, d_i) \}_{i \in \Vcal}$, where $\xb_i$ are the feature vectors in the training set and $d_i = +1$ if $i \in \Scal^{*}$ and $d_i = -1$ 
otherwise. 
Then, as long as this model does not make mistakes on the training set, 
one can readily conclude that the samples assigned to the classifier during training are as if they were \emph{sampled} from the feature 
distribution $p(\xb) \pi(d=-1 \,|\, \xb)$ induced by $\pi$. 
As a direct consequence, if the model $\pi(d \,|\, \xb)$ is smooth with respect to $\xb$, one can further conclude that the trained margin-based classifier will work \emph{well} on the \emph{unseen} samples it needs to decide upon at test time, \ie, samples from $p(\xb) \pi(d=-1 \,|\, \xb)$.
%

\section{Algorithms with Approximation Guarantees for Support Vector Machines}
\label{sec:algorithm}
In this section, we show that, for support vector machines (SVMs), the optimization problem defined in Eq.~\ref{eq:optimization-problem} can be rewritten as a maximization of the difference of two functions $g-c$, where $g$ is monotone, non-negative, and $\gamma$-weakly submodular and $c$ is non-negative modular. 
Moreover, we further show that the submodularity ratio $\gamma$ can be lower bounded and, as a consequence, a recently introduced deterministic greedy algorithm~\cite{harshaw2019submodular} as well as a more efficient randomized variant of the algorithm enjoy approximation guarantees at sol\-ving the problem.

\oursub{Monotonicity and weak submodularity} 
For (soft margin) SVMs, we can first rewrite the minimization problem defined in Eq.~\ref{eq:optimization-problem} as follows\footnote{\scriptsize In 
Appendix~\ref{app:hard-margin-linear-svm}, we also consider hard margin linear SVMs, which are relevant whenever the data is linearly separable.}:
\begin{align} \label{eq:soft-lin-svm}
\underset{\Scal, \wb, b}{\text{minimize}} & \quad \sum_{i \in \Vcal \backslash \Scal} \explain{\left[ \lambda \|\wb\|^2 + [1-y_i (\wb^\top\tf{\xb_i}+b) ]_{+} \right]}{\ell(h_{\wb, b}(\xb_i), y_i)}  
  + \sum_{i \in \Scal} \explain{[1-y_i h(\xb_i)]_{+}}{c(\xb_i, y_i)}    \\
\text{subject to} & \quad | \Scal | \leq n,\nn
\end{align}
%
where $\tf{\cdot}$ denotes a given feature transformation, $h(\cdot) \in [-H, H]$ is the (normalized) score provided by the human experts, which is 
only known for the training samples, and $H > 0$ is a given constant.   
In the above, we measure the human error $c(\xb, y)$ using a hinge loss $[1- y\cdot h(\xb)]_+$ because the SVM formulation also uses a hinge loss 
$[1 - y \cdot (\wb^{T} x + b)]_+$ to measure the machine error.
This is necessary in order to compare human and machine performance meaningfully. 
%
However, our solution is agnostic to this specific choice---it is applicable to any human error model.

Now, for any given set $\Scal$, let $\wb^{*}(\Vcal \backslash \Scal)$ and $b^{*}(\Vcal \backslash \Scal)$ be the parameters that minimize the objective
function above, \ie, $\wb^{*}(\Vcal \backslash \Scal), b^{*}(\Vcal \backslash \Scal) = \argmin_{\wb, b} \sum_{i\in\Vcal \backslash \Scal} [\lambda \|\wb\|^2 + (1-y_i (\wb^{\top} \tf{\xb_i} + b))_+]$. 
Here, note that these parameters can be found in polynomial time since the first two parameters in the objective function are convex. 
Then, we can rewrite the above minimization problem as a set function maximization problem:
\begin{equation} \label{eq:soft-lin-svm-2} 
\underset{\Scal}{\text{maximize}} \ \ \ g(\Scal)- c(\Scal), \qquad \text{subject to} \quad | \Scal | \leq n,
\end{equation}
where 
\begin{equation} \label{eq:g-def}
\begin{split}
g(\Scal) &=   \lambda |\Vcal|  \|\wb^*(\Vcal)\|^2  + \sum_{i \in \Vcal} [1  -   y_i (\wb^{*}(\Vcal)^\top \tf{\xb_i}+b^{*}(\Vcal))]_{+} \nn\\
&\quad -  \lambda |\Vcal \backslash \Scal |   \|\wb^*(\Vcal \backslash \Scal)\|^2  - \sum_{i \in \Vcal \backslash \Scal} [1 - y_i (\wb^{*}(\Vcal \backslash \Scal)^\top \tf{\xb_i}+b^{*}(\Vcal \backslash \Scal)) ]_{+},  
\end{split}
\end{equation}
and
\begin{equation}
c(\Scal) =  \sum_{i \in \Scal} [1-y_i h(\xb_i)]_{+}.
\end{equation}
%
%
%
In the above, the first term $\lambda |\Vcal| \|\wb^*(\Vcal)\|^2 + \sum_{i \in \Vcal} [1-y_i (\wb^{*}(\Vcal)^\top \tf{\xb_i}+b^{*}(\Vcal)) ]_{+}$ ensures that the function 
$g(\Scal)$ is non-negative and the function $c(\Scal)$ is clearly non-negative and modular\footnote{\scriptsize A set function $f(\Scal)$ is modular iff it satisfies that $f(\Scal \cup \{j \}) - f(\Scal) = f(\Lcal \cup \{j\}) - f(\Lcal)$ for all $\Scal \subseteq \Lcal \subseteq \Vcal$ and $j \in \Vcal$.}.
%
Moreover, we have the following proposition, which shows that $g(\Scal)$ is a monotone function (proven in Appendix ~\ref{app:mont}):
%
\begin{proposition}\label{thm:mont}
The set function $g(\Scal)$, defined in Eq.~\ref{eq:g-def}, is monotone, \ie, $g(\Scal \cup \{ j \}) - g(\Scal) \geq 0$ for all $\Scal \subseteq \Vcal$ and $j \in \Vcal$.
\end{proposition}
\begin{figure}[t!]
    \centering
  \subfloat[$\Ccal^{\pm}$ and $\Ccal^{\pm}_{1/s}$]{\includegraphics[width=0.26\textwidth]{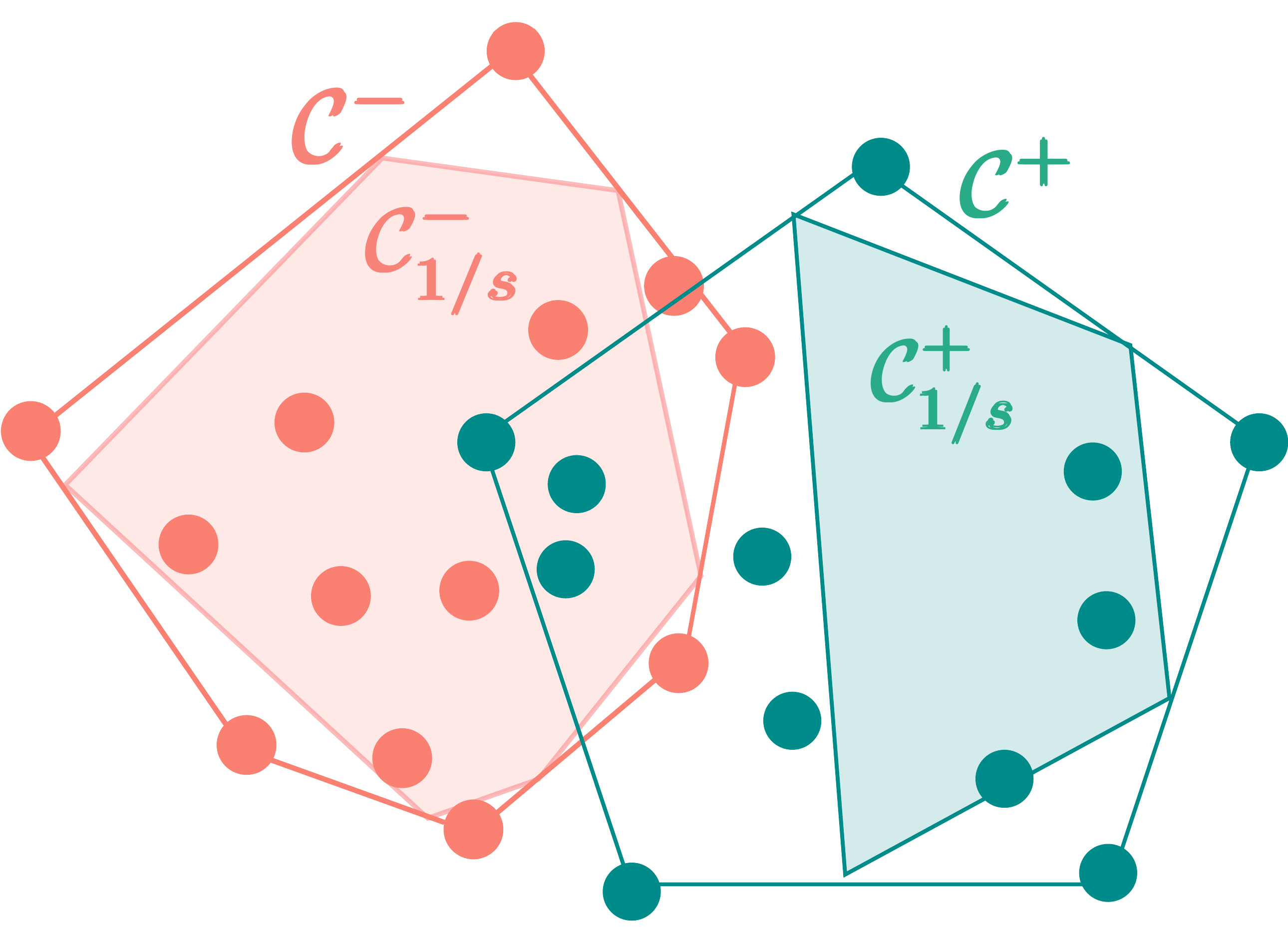}}
\hspace{16mm}
\subfloat[Sequence of $\Ccal^{\pm} _{1/s}$ for $s\in \NN^+$]
{\makebox[14em]{\includegraphics[width=0.26\textwidth]{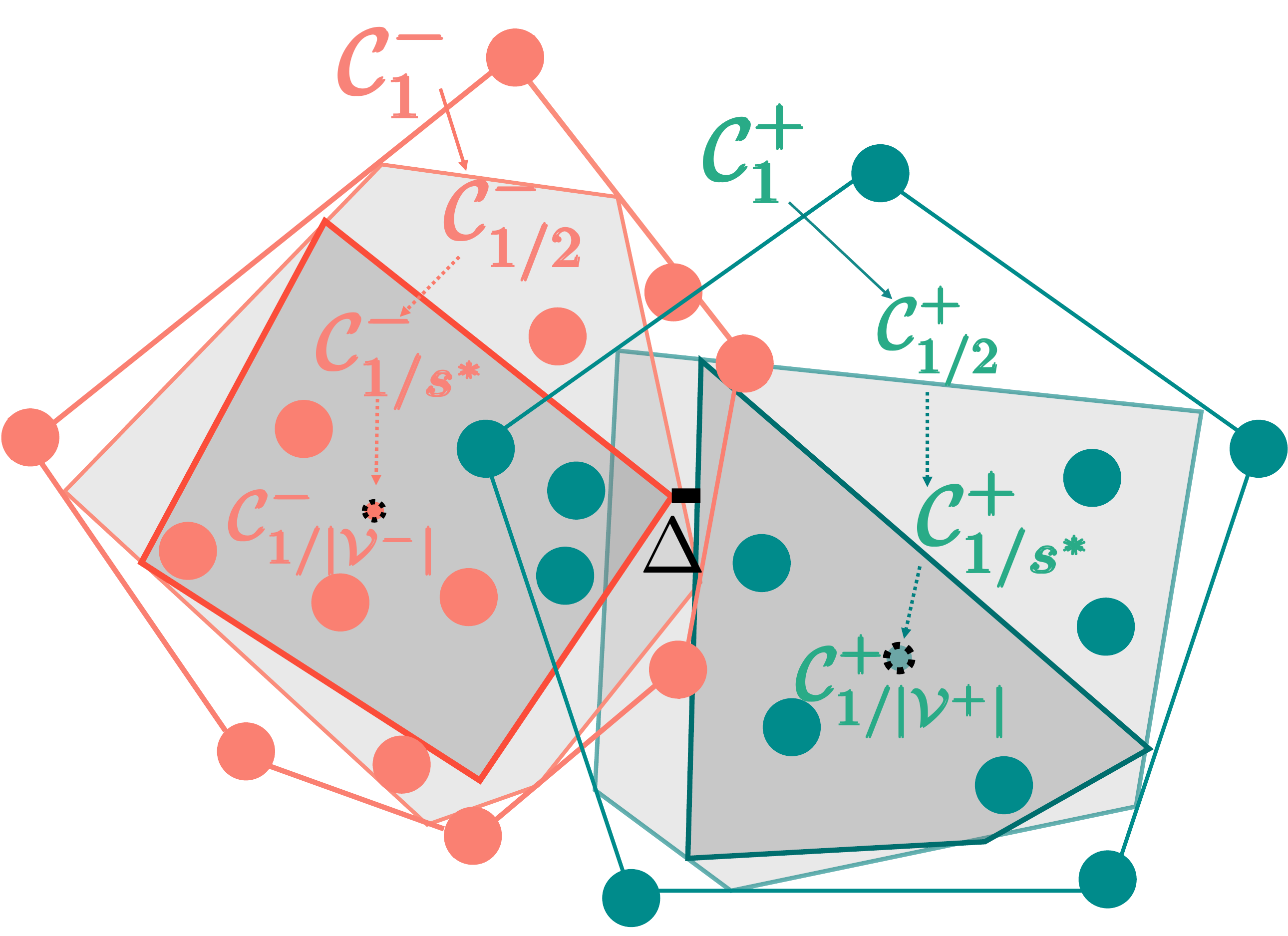}}}
    \caption{Convex hulls and reduced convex hulls. 
    %
    Panels (a) and (b) show the convex hulls $\Ccal^{\pm}$ and reduced convex hulls $\Ccal^{\pm}_{1/s}$ of a training set whose feature vectors 
    are non separable.
    In both panels, cyan and orange dots represent feature vectors $\xb_i$ with $y_i = 1$ and $y_i = -1$, respectively. 
    }
    \label{fig:ch}
\end{figure}

If the feature vectors in the training set are not separable according to their class labels, there exist instances of the problem in which the function $g(\Scal)$ has
sub\-mo\-du\-la\-ri\-ty ratio $\gamma = 0$\footnote{\scriptsize  A set function $f(\Scal)$ is $\gamma$-weakly submodular iff we have $\sum_{j \in \Lcal \backslash \Scal} 
[f(\Scal \cup \{j \}) - f(\Scal)] \geq \gamma [ f(\Scal \cup \Lcal) - f(\Scal) ]$ $\forall \Scal, \Lcal \subseteq \Vcal$ and $j \in \Vcal$. The largest $\gamma \leq 1$ such 
that the inequality is true is called submodularity ratio.}. 
However, we will now identify a general class of feature distributions for which the function $g(\Scal)$ has a nonzero submodularity ratio and its value can be lower 
bounded.
This lower bound will allow a recently introduced deterministic greedy algorithm as well as its randomized va\-riant to enjoy approximation guarantees at sol\-ving 
the pro\-blem. In the remainder, we first consider linear SVMs, \ie, $\tf{\xb} = \xb$, and then nonlinear SVMs.
%

%
Let $\Vcal^+$ and $\Vcal^-$ be the set of training samples with po\-si\-tive and negative labels, respectively, $\Ccal^+$ and $\Ccal^-$ be their 
corresponding convex hulls, \ie,
\begin{equation}
\Ccal^{\pm} = \left\{ \sum_{i \in\Vcal^{\pm}} \mu_i\xb_i \, \Big| \, \sum_{i\in\Vcal^{\pm}}\mu_i=1,\ \mu_i \ge 0 \right\},
\end{equation}
and $\Delta$ be the minimum distance between them, \ie, $\Delta = \min_{\xb^{+} \in \Ccal^{+},\, \xb^{-} \in \Ccal^{-}} \|\xb^+-\xb^-\|_2$.
Then, note that, whenever the feature vectors in the training set are not separable according to their class labels, the above convex hulls 
overlap and thus $\Delta = 0$.
However, under mild technical conditions, there will always exist subsets of feature vectors within these convex hulls that do not overlap, 
as shown in Figure~\ref{fig:ch}. 
To characterize these subsets, we introduce the notion of reduced convex hulls~\cite{bennett2000duality}:
\begin{equation}
\Ccal^{\pm}_{1/s} = \left\{\sum_{i\in\Vcal^{\pm}} \mu_i\xb_i \, \Big| \, \sum_{i\in\Vcal^{\pm}}\mu_i=1,\ 0 \le \mu_i \le \frac{1}{s} \right\},\label{eq:cnp}
\end{equation}
where $s \in \NN^+$, with $s \leq \min\{|\Vcal^+|,|\Vcal^-|\}$ and, similarly as before, we denote the minimum distance between them as 
$\Delta_{1/s} = \min_{\xb^{+}\in\Ccal^{+}_{1/s}, \, \xb^{-}\in\Ccal^{-}_{1/s}}\|\xb^+-\xb^-\|$.

Now, consider the sequence of reduced convex hulls $\{ (\Ccal^{+} _{1/s},\Ccal^{-} _{1/s}) \}_{s=1}^{\Vcal_{\min}}$, where $\Vcal_{\min} = \min\{|\Vcal^+|,|\Vcal^-|\}$,
illustrated in Figure~\ref{fig:ch}(b), and note that the corresponding minimum distances, by definition, satisfy that $\Delta_{1/s} \leq \Delta_{1/s'}$ for all $s' > s$. 
%
%
Then, we measure to what extent feature vectors with positive and negative labels overlap using the distance 
\begin{equation} 
\Delta^* = \min_{s \in \{1, \ldots, \Vcal_{\min}\}} \ \{\Delta_{1/s} \ \vert \ \Delta_{1/s}>0\},~\label{eq:deltas}
\end{equation}
where the higher the value of $s^* = \argmin_{s} \ \{\Delta_{1/s} \ \vert \ \Delta_{1/s}>0\}$, the higher the overlap between feature vectors with positive and negative labels.
Moreover, if there are no elements in the sequence with positive distance, we set $\Delta^{*} = 0$ and $s^{*} = 0$.
%

Given the above, we are now ready to present and prove one of our key results, which characterizes the submodularity ratio of the function $g(\Scal)$ in terms
of the amount of overlap between feature vectors with positive and negative labels, as measured by the distance $\Delta^{*}$ (proven in Appendix~\ref{app:thm:soft-linear-subm}):
%
%
\begin{theorem} \label{thm:soft-linear-subm}
 Let $\tf{\xb} = \xb$,
 $\rho^*=\Vcal_{\min}/{|\Vcal|}$, $\sigma^*={s^*}/{|\Vcal|}$,
$\eta= {\left(2\sqrt{\lambda}+\max_{i\in\Vcal}\bnm{\xb_i}\right) }/{\sqrt{\lambda}}$. Then, the 
submodularity ratio $\gamma$ of the function $g(\Scal)$ (defined in Eq.~\ref{eq:g-def}) satisfies that
\begin{align} 
\gamma\ge 
\gamma^*= 
\dfrac{\min\left\{ \dfrac{\left[\Delta^{\ast}\sigma^* \right]^2}{4\lambda},  \dfrac{1}{(\eta-2)^2} \right\}}
 {\eta  +  \dfrac{\eta^2}{2} \left( \dfrac{1}{2} 
 + \sqrt{\dfrac{1}{4} +\dfrac{4 |\Vcal|   (\eta-1)}{\eta^2}} \right) +  \left(\eta-1\right) |\Vcal| } \nn 
 \end{align}
as long as the number of samples outsourced to humans $n\le ( \rho^*-\sigma^*)|\Vcal|$. 
\end{theorem}

For nonlinear SVMs, rather than characterizing the submodularity ratio in terms of $\Delta^{*}$ and $s^{*}$, we resort to the
spectral properties of the kernel matrix $\Kbb=[K(\xb_i,\xb_j)]_{i,j\in\Vcal}$. In particular, our key result is the fo\-llo\-wing Theorem (proven in Appendix~\ref{app:soft-nonlinear-subm}):
%
\begin{theorem} \label{thm:soft-nonlinear-subm}
Let  $\eta= {\left(2\sqrt{\lambda}+\max_{i\in\Vcal}\sqrt{K(\xb_i,\xb_i)} \right) }/{\sqrt{\lambda}}$, $\Yb=\diag(\{y\}_{i\in\Vcal})$, $\rho^*=\Vcal_{\min}/{|\Vcal|}$, and 
\begin{equation*}
\displaystyle{\zeta=\minz{\mub \ge 0, \sum_{i\in\Vcal^{+}}\mu_i=1, \sum_{i\in\Vcal^{-}}\mu_i=1}\  \mub^\top\Yb^\top \KK \Yb  \mub}, 
\end{equation*}
If the kernel matrix is full rank, \ie, $\text{rank}(\Kbb)=|\Vcal|$, then the submodularity ratio $\gamma$ of the function $g(\Scal)$ satisfies that
 \begin{align} 
\gamma\ge 
\gamma^*= 
\dfrac{\min\left\{\dfrac{\zeta \sigma^{*2}}{4\lambda}, \dfrac{1}{(\eta-2)^2} \right\}}
 {\eta  +  \dfrac{\eta^2}{2} \left( \dfrac{1}{2} 
 + \sqrt{\dfrac{1}{4} +\dfrac{4 |\Vcal|   (\eta-1)}{\eta^2}} \right) +  \left(\eta-1\right) |\Vcal| }\nn
 \end{align}
as long as $n\le ( \rho^*-\sigma^*)|\Vcal|$ for some $\sigma^* \in (0,\rho^*]$.
\end{theorem}
Finally, for the particular case of (soft margin) SVMs without offset, \ie, $b = 0$ in Eq.~\ref{eq:soft-lin-svm}, we can derive a stronger 
lower bound, which does not depend on $\Delta^{*}$ and $s^{*}$, by exploiting their greater stability properties
(proven in 
Appendix~\ref{app:soft-linear-without-offset-subm}): 
\begin{theorem} \label{thm:soft-linear-without-offset-subm}
If $\eta= {\left(2\sqrt{\lambda}+\max_{i\in\Vcal} \sqrt{K(\xb_i, \xb_i)}\right) }/{\sqrt{\lambda}}$, then for SVMs without offset, ($b=0$ in Eq.~\ref{eq:soft-lin-svm}), the submodularity ratio of the function $g(\Scal)$ is given by
\begin{equation*}
\gamma\ge 
\gamma^*= 
 {  \min\left\{   \dfrac{1}{(\eta-2)^2}, \dfrac{1}{2} \right\} }\Big/
 {  \left[\eta  +   \dfrac{\eta^2}{2}\right]   }.
 \end{equation*} 
\end{theorem}

%
%
\xhdr{Proof sketch of our key technical results} The proofs of Theorems~\ref{thm:soft-linear-subm},~\ref{thm:soft-nonlinear-subm} and~\ref{thm:soft-linear-without-offset-subm}
consist of two steps. 
In the first step, they show that $g(\Scal\cup \{j\})-g(\Scal) \ge \lambda \bnm{\wb^*(\Vcal\cp\Scal)}^2$ and use the dual formulation of SVM as well as the properties of the corresponding 
SVM model to derive a lower bound of $\bnm{\wb^*(\Vcal\cp\Scal)}$. 
In the second step, they use the stability pro\-per\-ty of SVM~\cite{bousquet2002stability} to derive the upper bound on
$g(\Scal\cup\Lcal)-g(\Scal)$. These two steps together lead to the bound on $\gamma^*$.

While the bounds in all the above theorems are tight in terms of the size of the dataset, there are notable differences 
between the different model classes. For SVMs with offsets, Theorems~\ref{thm:soft-linear-subm} and~\ref{thm:soft-nonlinear-subm} suggest that the submodularity ratio bound 
decreases proportionally to $1/|\cal{V}|$ and this is due to their poor stability properties~\cite{hush2007stability}. For SVMs without offsets, Theorem~\ref{thm:soft-linear-without-offset-subm} tells us that the submodularity 
ratio bound is independent of $|\cal{V}|$ and this is due to their greater stability properties~\cite{bousquet2002stability}. More specifically, for any $\Lcal \subseteq \Vcal \backslash \Scal$, 
the marginal gain $g(\Scal\cup\Lcal) - g(\Scal)$ can be upper bounded by a smaller quantity than in the case of SVMs with offsets and this results into a stronger lower bound.

\begin{algorithm}[t]
\SetAlgoNoLine
  \mbox{\textbf{Input}: Ground set $\Vcal$, functions $g$ and $c$, parameters $n\text{ and }\gamma$}.\\
  \textbf{Initialize}: $\Scal \leftarrow \varnothing $\\
  \For{$i=0,\ldots,n-1$}{
  $\omega_i\leftarrow \left(1-\frac{\gamma}{n}\right)^{n-(i+1)}$\\
    $  k^{*} \leftarrow \displaystyle{\argmax_{k\in\Vcal \setminus \Scal}}  \left\{ \textstyle{\omega_i\cdot\left[ g(\Scal \cup \{k\}) - g(\Scal)  \right]}      - c(\{k\}) \right\}$\\
   \If{$\omega_i\cdot\left[ g(\Scal \cup \{k^{*}\}) - g(\Scal) \right] - c(\{k^{*}\}) > 0$}{
    $\Scal \leftarrow \Scal \cup \{k^{*}\} $\
    }
 }
\textbf{Return} $\Scal$  
\caption{Distorted greedy algorithm} \label{alg:greedy}
\end{algorithm}

\xhdr{Distorted greedy algorithm} 
%
%
The distorted greedy algorithm~\cite{harshaw2019submodular} proceeds iteratively and, at each iteration, it assigns to the humans the sample $(\xb_k,y_k)$ that provides the 
highest marginal \emph{distorted} gain among the remaining training samples $\Vcal \setminus \Scal$.
Algorithm~\ref{alg:greedy} summarizes the algorithm, which requires the value of the submodularity ratio $\gamma$ as an input. Since we only have data dependant bounds on
$\gamma$, in our experiments, we use the meta algorithm proposed in~\citet{harshaw2019submodular} to \emph{guess} the value of 
$\gamma$.

Since we have shown that the objective function in Eq.~\ref{eq:soft-lin-svm-2} can be expressed as the difference of two functions $g - c$, where g is monotone, non-negative and $\gamma$-weakly
submodular and $c$ is non-negative modular, it readily follows from Theorem 3 in~\citet{harshaw2019submodular} that the distorted greedy algorithm enjoys approximation 
guarantees. More specifically, the distorted greedy algorithm is guaranteed to return a set $\Scal$ such that
\begin{equation}
	g(\Scal)-c(\Scal) \geq (1-e^{-\gamma})g(\texttt{OPT}) - c(\texttt{OPT}),
\end{equation}
where $\texttt{OPT}$ is the optimal set and $\gamma \leq \gamma^{*}$ with $\gamma^{*}$ defined in Theorem~\ref{thm:soft-linear-subm} (linear SVMs with offset), Theorem~\ref{thm:soft-nonlinear-subm} (nonlinear SVMs with offset), or Theorem~\ref{thm:soft-linear-without-offset-subm} (SVMs without offset).

In addition to the distorted greedy algorithm, which needs to make $O(n |\Vcal|)$ evaluations of the function $g(\cdot)$, ~\citet{harshaw2019submodular}  has also proposed a randomized variant of
the algorithm, which enjoys an asymptotically faster run time due to the use of sampling techniques and is also applicable to our problem. Instead of optimizing over the entire ground 
set $\Vcal$ at each iteration, it optimizes over a random sample $\Bcal_i \subseteq \Vcal$ of size $O(\frac{|\Vcal|}{n} \log \frac{1}{\epsilon})$. Hence, it only needs to make 
$O(|\Vcal| \log(\frac{1}{\epsilon}))$ evaluations of $g(\cdot)$ and it returns a set $\Scal$ such that
%
\begin{equation}
	\EE[g(\Scal)-c(\Scal)] \geq (1-e^{-\gamma}-\epsilon)g(\texttt{OPT}) - c(\texttt{OPT}).
\end{equation}
In the next sections, we will demonstrate that, in addition to enjoying the above approximation guarantees, the distorted greedy algorithm as well as its randomized variant
perform better in practice than several competitive baselines.

\section{Experiments on Synthetic Data}
\label{sec:synthetic}
%
%
%
In this section, we first look into the solutions provided by the distorted greedy algorithm (Alg.~\ref{alg:greedy}) in a variety of synthetic examples.
%
Then, we compare the performance of the distorted greedy algorithm and its randomized variant with several competitive 
baselines.
%
%
Finally, we investigate the influence that the amount of human error has on the number of samples outsourced to 
humans by the distorted greedy algorithm. 
%
%
\begin{figure*}[t!]
    \centering
	 {\includegraphics[width=0.6\textwidth]{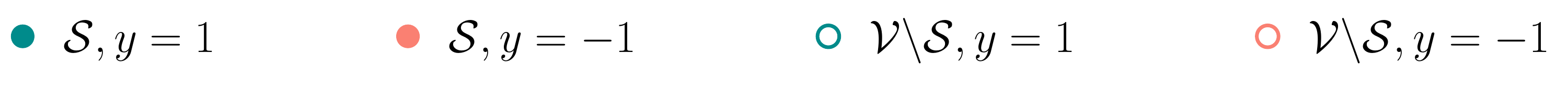}\label{fig:ltr}}\\
     \subfloat[Linear SVM]{
        \scriptsize
	\stackunder[5pt]{\includegraphics[width=0.23\textwidth]{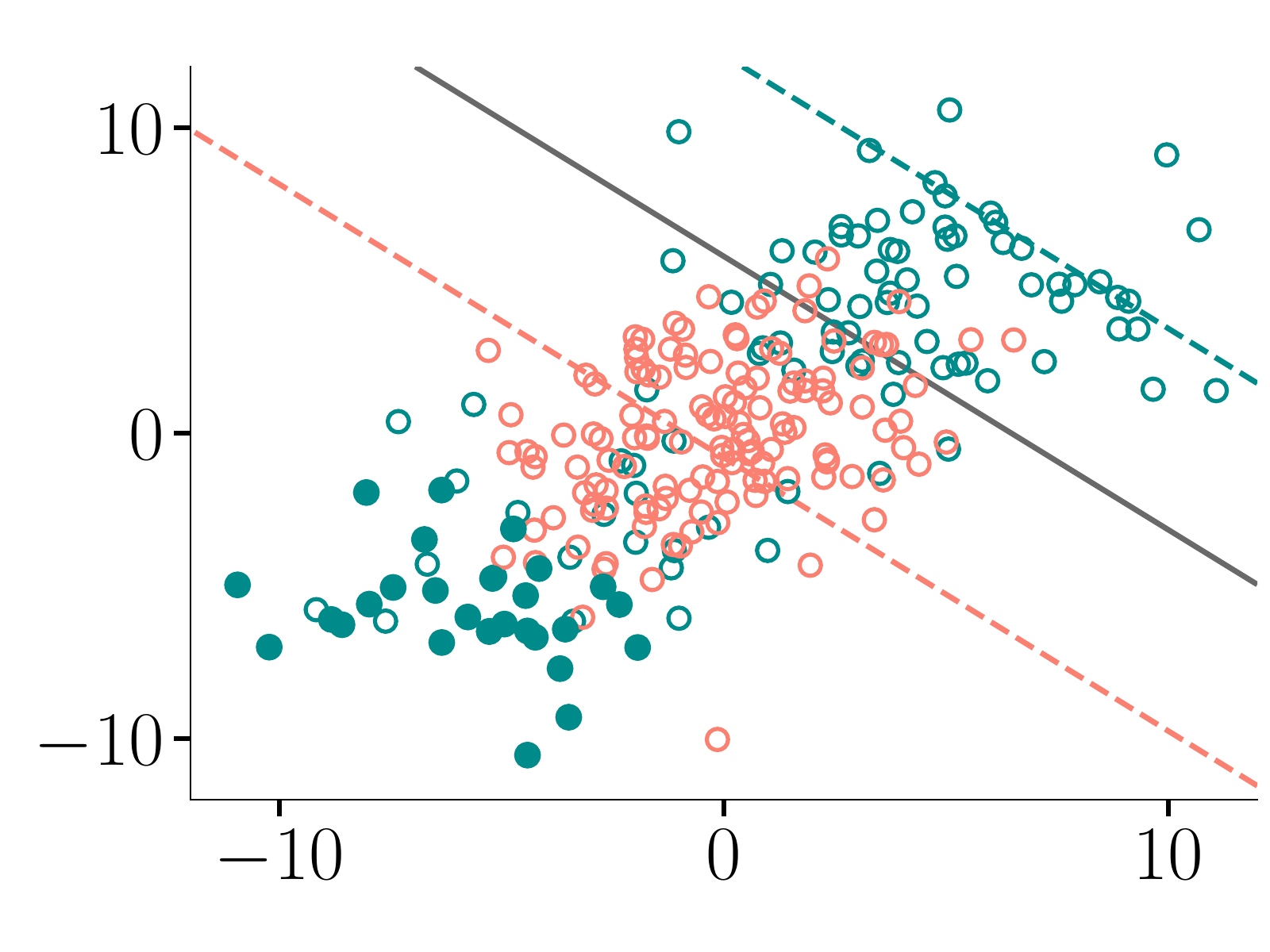}}{$n/|\Vcal| = 0.12$}
	\stackunder[5pt]{\includegraphics[width=0.23\textwidth]{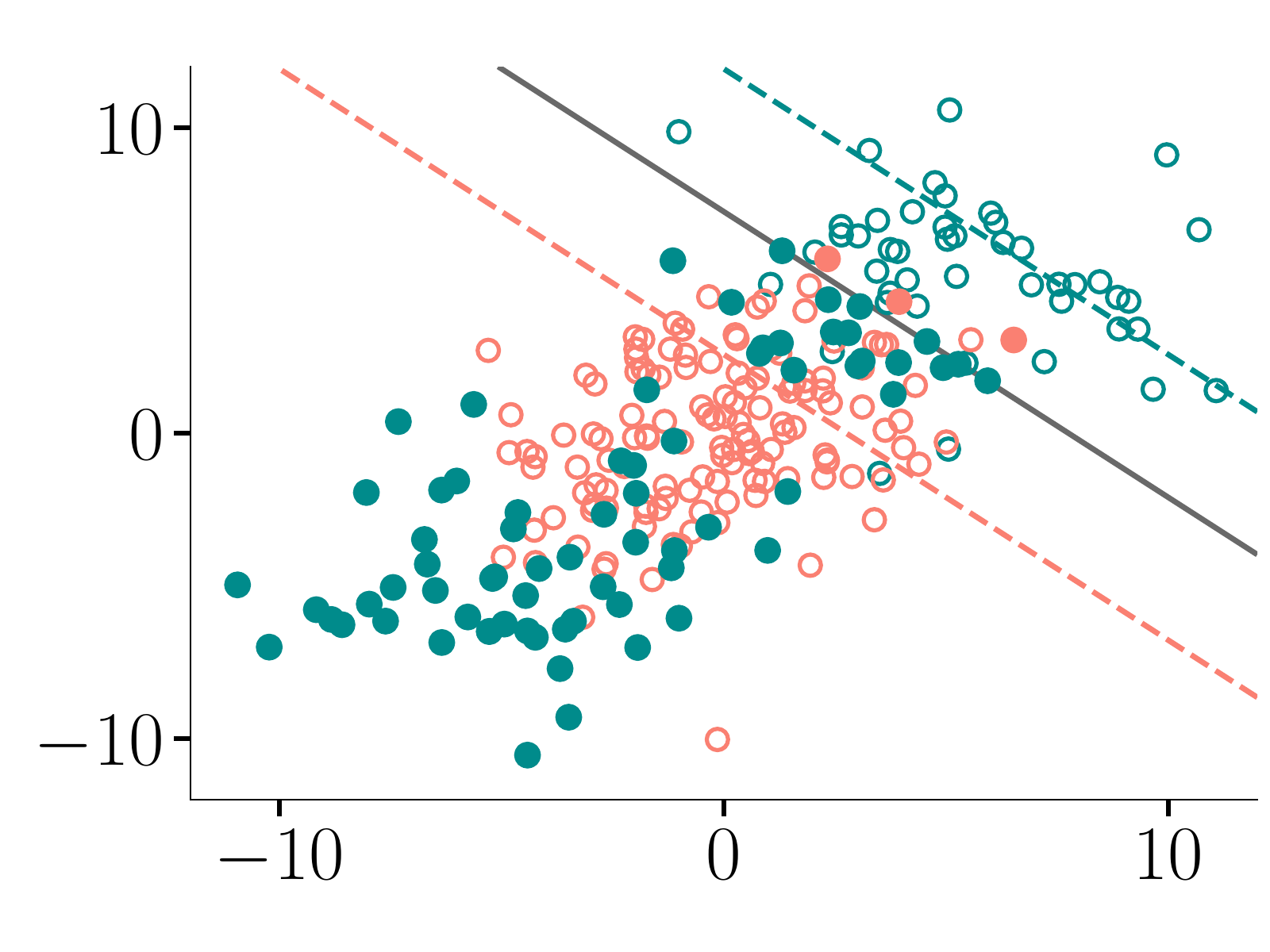}}{$n/|\Vcal| = 0.3$}
     } \hspace{8mm}
     \subfloat[Nonlinear SVM]{
        \scriptsize
	\stackunder[5pt]{\includegraphics[width=0.23\textwidth]{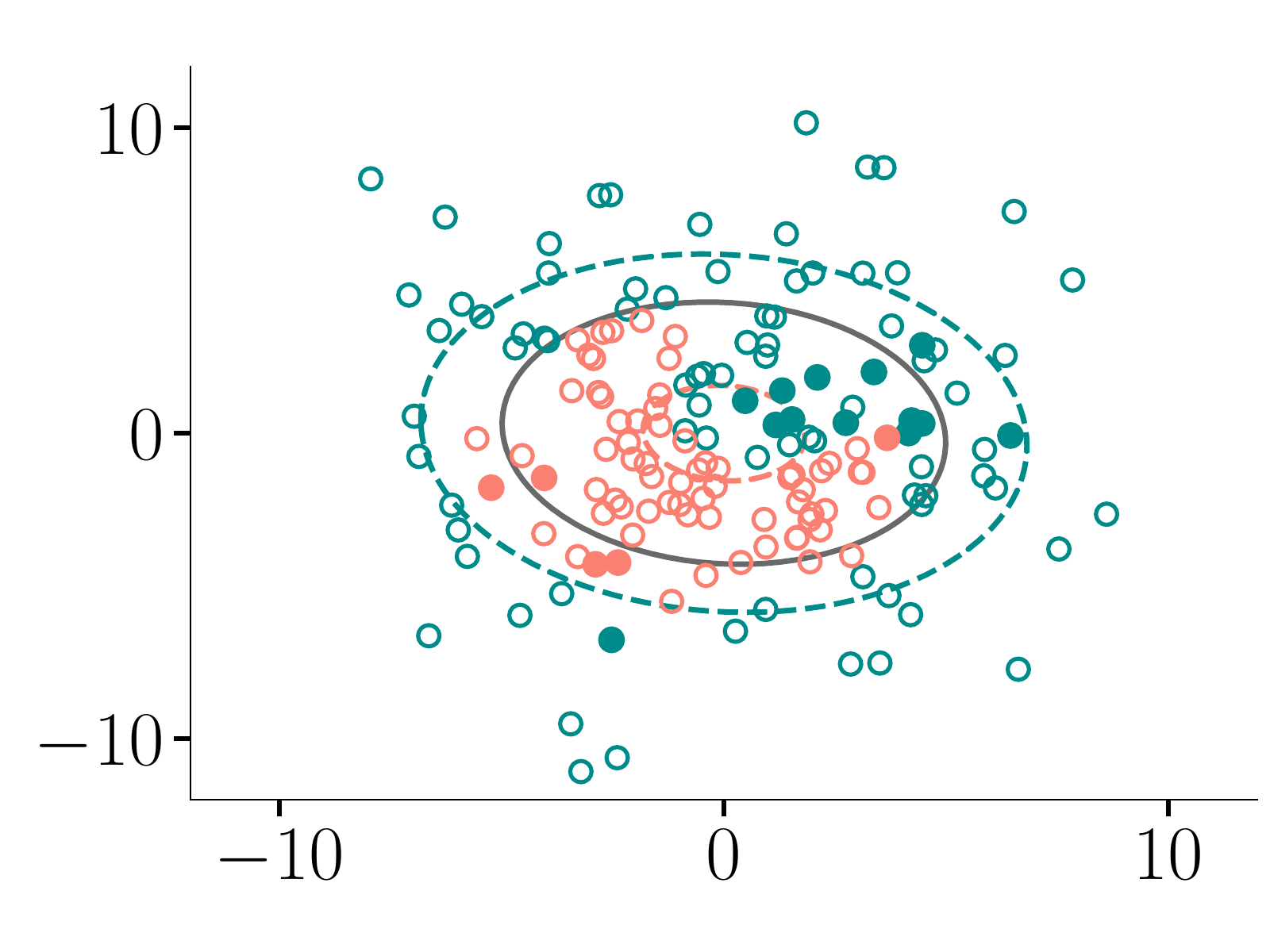}}{$n/|\Vcal| = 0.12$}
	\stackunder[5pt]{\includegraphics[width=0.23\textwidth]{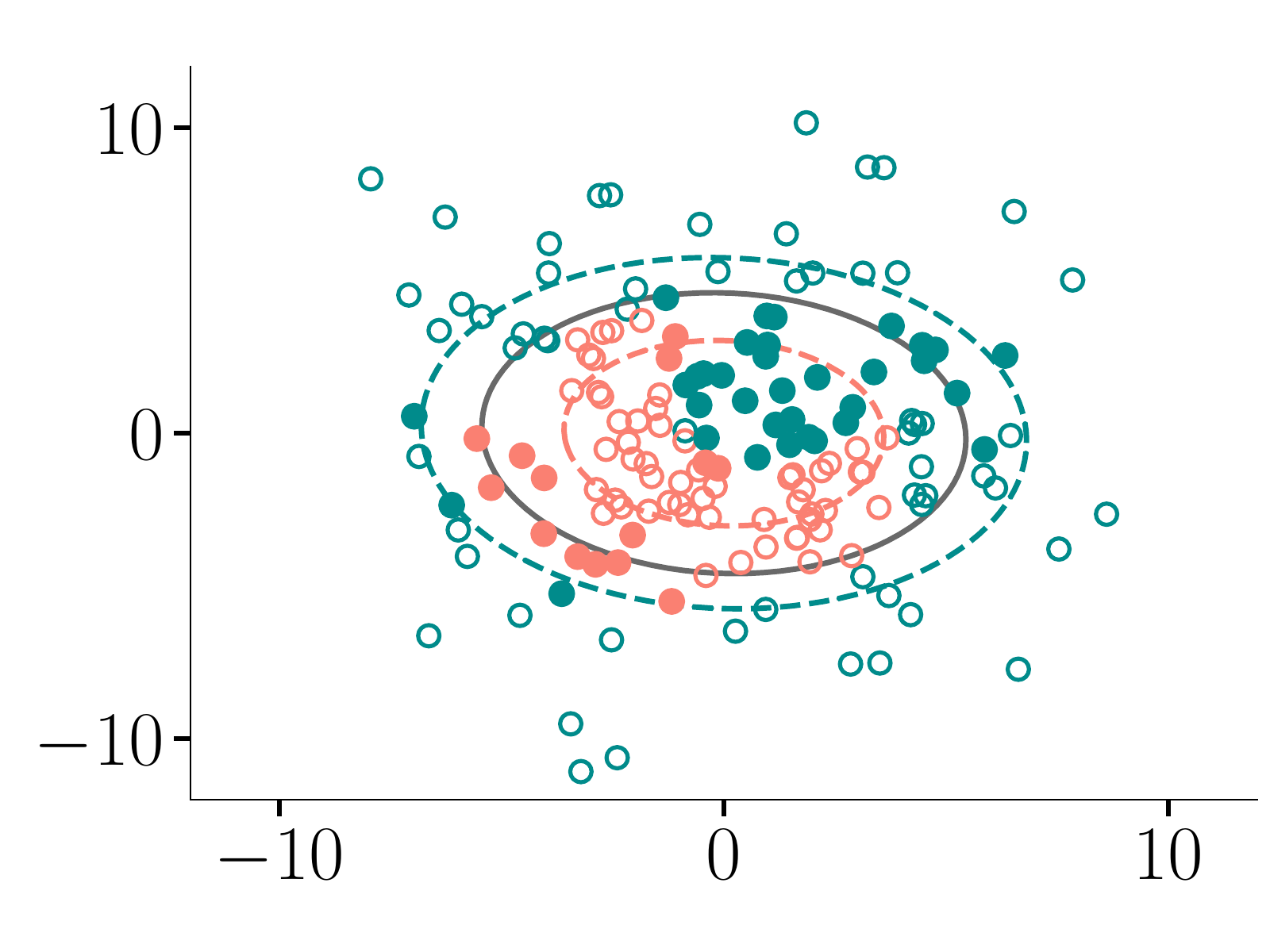}}{$n/|\Vcal| = 0.3$}
      }
\caption{Linear and non linear support vector machines with offset under human assistance trained using the stochastic greedy algorithm (Alg.~\ref{alg:greedy}). 
In each panel, circles represent the feature vectors in the training set, filled and empty circles are assigned to humans and machines, respectively, 
the solid line indicates $\wb^{*\top}(\Vcal\cp \Scal^*)  \phi(\xb) +b(\Vcal\cp \Scal^*)=0$, and the cyan and orange dashed lines indicate 
$\wb^{*\top}(\Vcal\cp \Scal^*)  \phi(\xb)  +b(\Vcal\cp \Scal^*)= 1$ and $\wb^{*\top}(\Vcal\cp \Scal^*)  \phi(\xb)+b(\Vcal\cp \Scal^*)= -1$, respectively.
For both linear and nonlinear SVMs, we used $\lambda = 1$, $\Delta H=0.2$ and, for nonlinear SVMs, we used a quadratic kernel $K(\xb_i,\xb_j) = (\frac{1}{2}\langle\xb_i,\xb_j\rangle)^2$. 
}
\label{fig:syn-dg-main}
\end{figure*}

\oursub{Experimental setup}
%
In all our experiments, we ge\-ne\-rate $|\Vcal|=400$ samples using a generative process under which the corresponding SVM under full
automation is unable to perform well. 
For linear SVMs with offset, we draw the class labels uniformly at random, \ie,  $y \sim \text{Bernoulli}(0.5)$, and then draw
the features from two diffe\-rent distributions, $p(\xb|y=-1) = \Ncal([0,0], [6, 1; 1, 6] )$ and $p(\xb|y=1) = \beta \Ncal([5,5], [6, 1; 1, 6] ) + (1-\beta) \Ncal([-5,-5], [6, 1; 1, 6] )$, where $\beta\sim\text{Bernoulli}(0.5)$.
For nonlinear SVMs, we draw the features from a single distribution $p(\xb) \sim \Ncal([0,0],  [12, 1; 1, 14] )$ and, for each sampled feature $\xb$, we set $y=+1$ if $\bnm{\xb-[1,1]}\le 2$ or $\bnm{\xb+[1,1]}\ge 5$ and $y=-1$, otherwise. Here, we use a quadratic kernel $K(\xb_i,\xb_j) = (\frac{1}{2}\langle\xb_i,\xb_j\rangle)^2$. 
%
%
Moreover, we generate the scores provided by human experts $h(\xb)$ per label by drawing samples from two uniform distributions, 
{\ie, $h(\xb) \sim \text{Unif}[-\Delta H,-\Delta H+1]$ if $y=1$ and $h(\xb) \sim \text{Unif}[-1+\Delta H,\Delta H]$ otherwise.} The exact 
choice of $\Delta H$ varies across experiments and is mentioned therein.
%
%
Finally, in each experiment, we use 60\% of samples for trai\-ning and 40\% of samples for testing, set $\lambda=1$ and we fo\-llow the
procedure described in Section~\ref{sec:formulation} to train a multilayer perceptron $\pi(\cdot|\xb)$ that decides which samples to 
outsource to humans at test time. 

\xhdr{Baseline methods} We compare the performance of the distorted greedy algorithm and its randomized variant with four 
competitive baselines:

\noindent\emph{---} Triage based on algorithmic uncertainty~\cite{raghu2019algorithmic}: it first trains a support vector machine for full automation, 
\ie, $\Scal = \emptyset$. Then, at test time, it outsources to humans the top $n$ testing samples sorted in decreasing order of the classification 
uncertainty of this support vector machine, defined as $1/|\wb^*(\Vcal)\phi(\xb)+b^*(\Vcal)|$. 

\noindent\emph{---} Triage based on predicted error~\cite{raghu2019algorithmic}: it trains a support vector machine for full automation, \ie, $\Scal = \emptyset$, and two additional supervised models that predict the human error and the support vector machine error. Then, at test time, it 
outsources to humans the top $n$ testing samples sorted in decreasing order of the difference between the predicted machine error and the 
predicted human error.

\noindent\emph{---} Full automation: it trains a support vector machine for full automation and, at test time, the trained support vector machine predicts 
all testing samples.

\noindent\emph{---} No automation: humans predict all testing samples.

We evaluate the performance of all the methods in terms of the misclassification test error $\PP(y\neq\hat{y})$ 
and F1 score on po\-si\-tive test samples.
%
Here, our F1 score is a useful metric in medical diagnosis, since it measures the ability of a model to detect the specific disease.  

\oursub{Results}
First, we look into the solutions $(\Scal^{*}, \wb^{*}(\Vcal\cp \Scal^*),$ $ b^{*}(\Vcal\cp \Scal^*))$ provided by the distorted greedy algorithm. 
Fi\-gure~\ref{fig:syn-dg-main} summarizes the results, which shows that:
(i) the distorted greedy algorithm outsources to humans those samples that are more prone to be misclassified by the machine, \ie, $\{(\xb,y)\,|\, y [ \wb^{*\top}(\Vcal\cp \Scal^*)  \phi(\xb) +b(\Vcal\cp \Scal^*)] < 0\}$; and,
(ii) the higher the samples $n$ outsourced to humans, the lower the number of training samples inside the margin, \ie, $\{(\xb,y) \,|\, |\wb^{*\top}(\Vcal\cp \Scal^*)  \phi(\xb) +b(\Vcal\cp \Scal^*)| \le 1\}$, among those samples the SVM needs to decide upon.

Next, we compare the performance of the distorted greedy algorithm and its randomized variant with all the baselines---algorithmic uncertainty triage, predicted error triage, 
full automation and no automation---in terms of the misclassification test error $\PP(y\neq\hat{y})$ 
and F1 score on po\-si\-tive test samples.
Figure~\ref{fig:baselines-syn} summarizes the results, which shows that both algorithms consistently outperform all the baselines by large margins.
%
%
%
Finally, we investigate the influence that the amount of human error has on the number of samples $|\Scal| \leq n$ outsourced to humans by the distorted greedy algorithm 
and its randomized variant. 
Figure~\ref{fig:s-vs-n} summarizes the results for the linear SVM with offset. We find that, as long as the amount of human error is small, both algorithms outsource 
$|\Scal| = n$ samples to humans, however, for higher levels of human error, the algorithm outsources fewer samples $|\Scal| < n$ since it is more beneficial for the
minimization of the overall error.
%
%
\begin{figure}[t!!]
\centering
\hspace{0.2cm}{\includegraphics[width=0.7\textwidth]{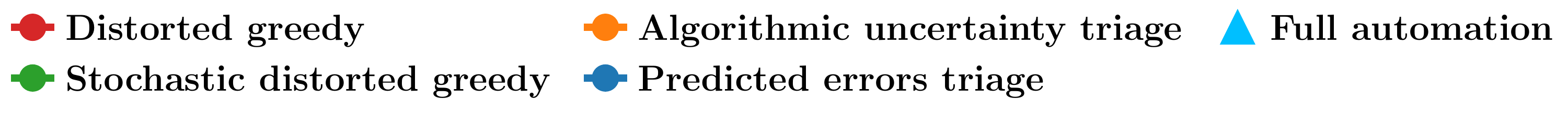}} \\[-2ex]
\subfloat[Misclassification error]{\includegraphics[width=0.27\textwidth]{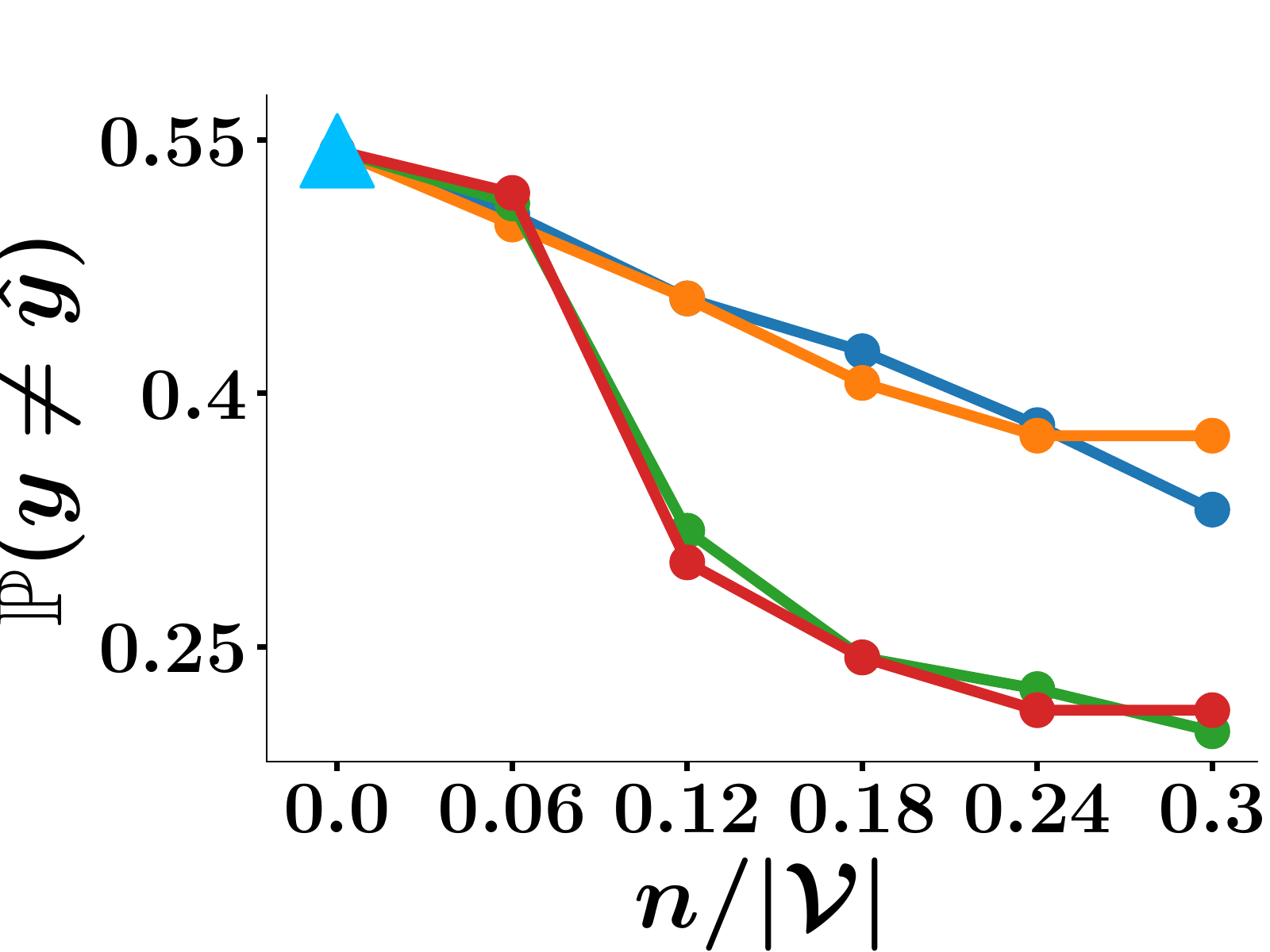}} \hspace{15mm}
\subfloat[F1 score]{\includegraphics[width=0.27\textwidth]{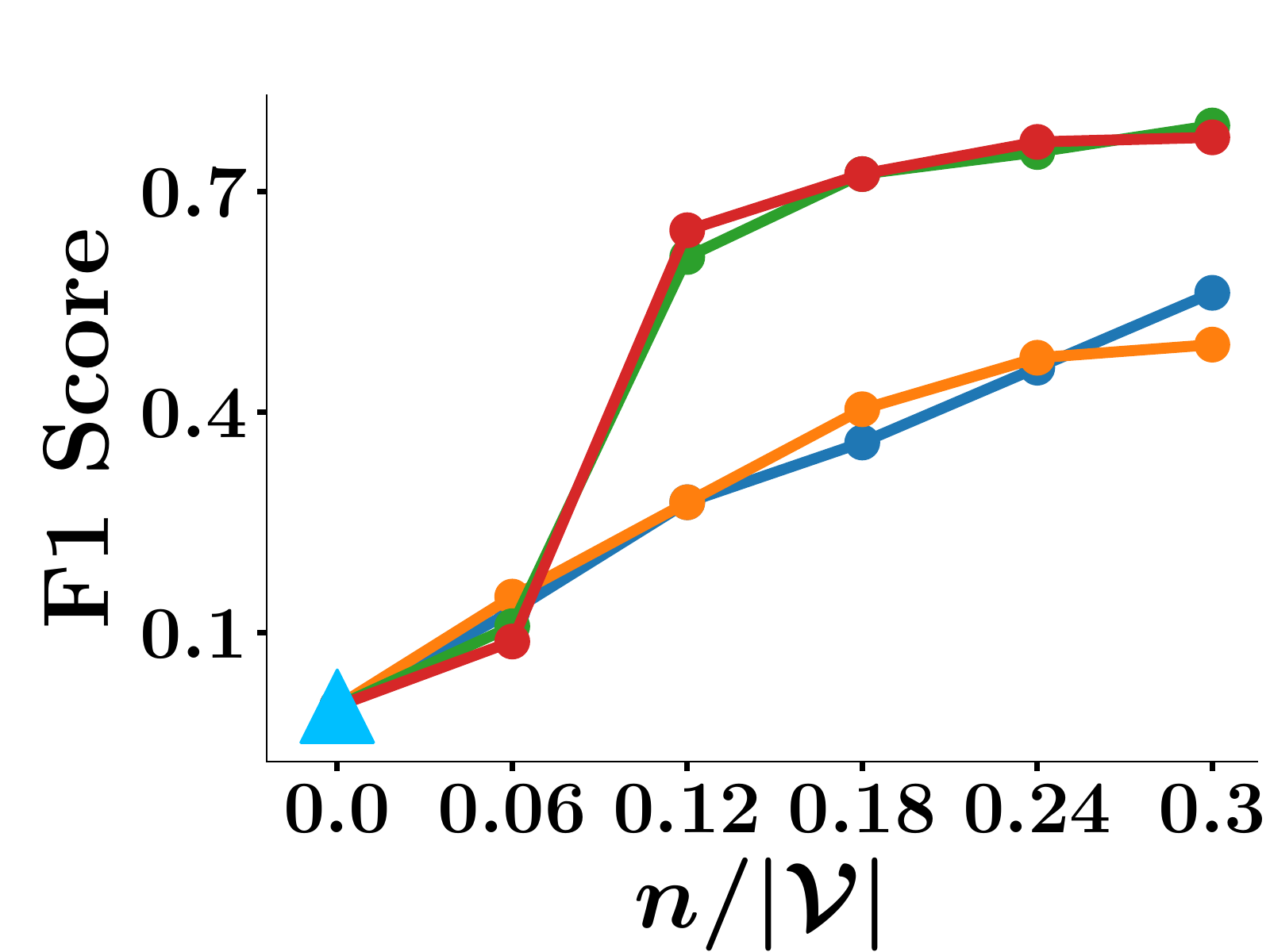}}\\
\caption{Performance of all methods 
on synthetic data with linear SVM with offset. For all methods, we set $\Delta H = 0.2$. If humans predict all testing samples (No automation), 
we have that  {$\PP(\hat{y}\neq y) = 0.19$ and $\text{F1-score} = 0.81$.}}
\label{fig:baselines-syn}
\end{figure}
\begin{figure}[t]
\centering
\hspace{0.0cm}{\includegraphics[width=0.7\textwidth]{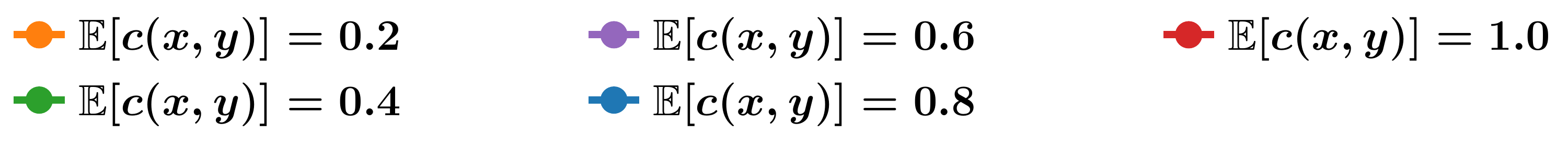}}\\[-2ex]
\subfloat[Distorted greedy]{\includegraphics[width=0.27\textwidth]{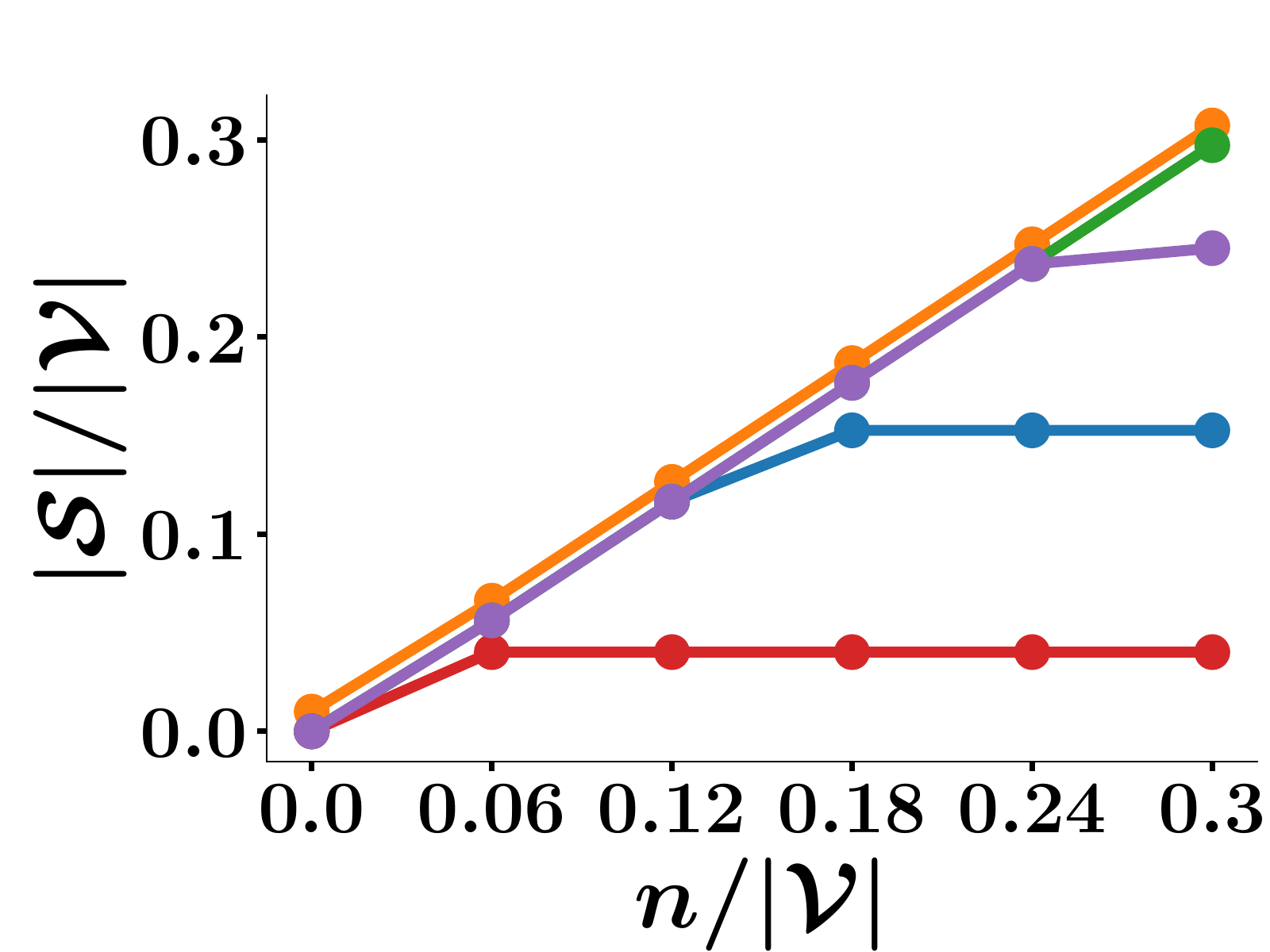}}\hspace{16mm}
\subfloat[\hspace{-0.3mm}Stochastic distorted greedy]{\makebox[12em]{\includegraphics[width=0.27\textwidth]{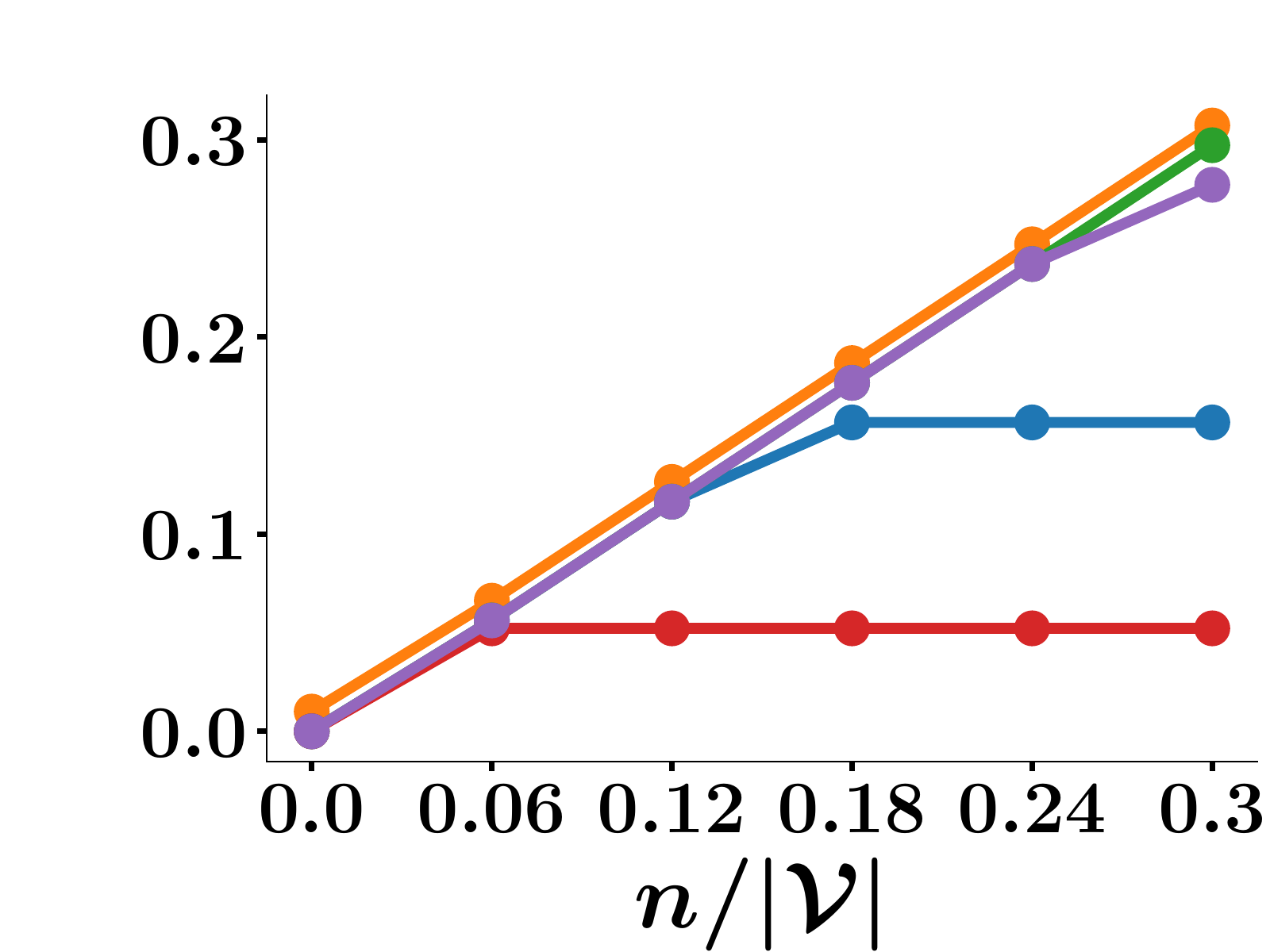}}}
\caption{Number of samples $|\Scal|$ outsourced to humans by the distorted greedy algorithm and its randomized variant against the maximum number of 
outsourced samples $n$ for different amount of human errors.}
\label{fig:s-vs-n}
\end{figure}

\section{Experiments on Real Data} 
\label{sec:real}
We experiment with three real-world medical datasets and first show that, under human assistance, support vector machines 
trained using the distorted greedy algorithm as well as its randomized variant outperform those trained for full automation as well 
as humans operating alone. 
Then, we investigate the influence that the amount of human error has on the performance of the distorted greedy algorithm. 

\begin{figure*}[t]
\centering
\hspace{0.1cm}{\includegraphics[width=0.7\textwidth]{fig/legend_notbold}}\\[-3ex]
\subfloat{\includegraphics[width=0.26\textwidth]{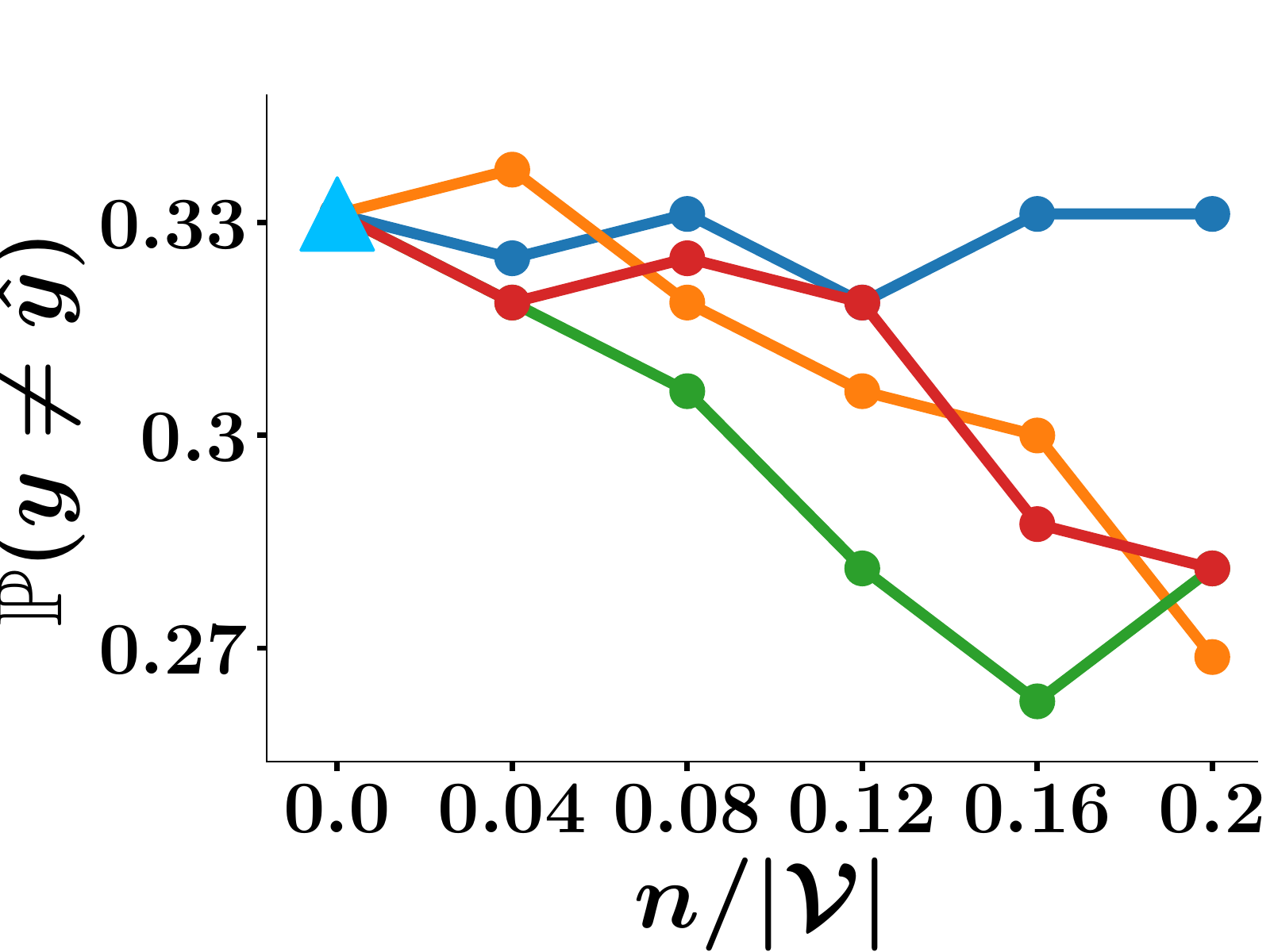}}\hspace{6mm}
\subfloat{\includegraphics[width=0.26\textwidth]{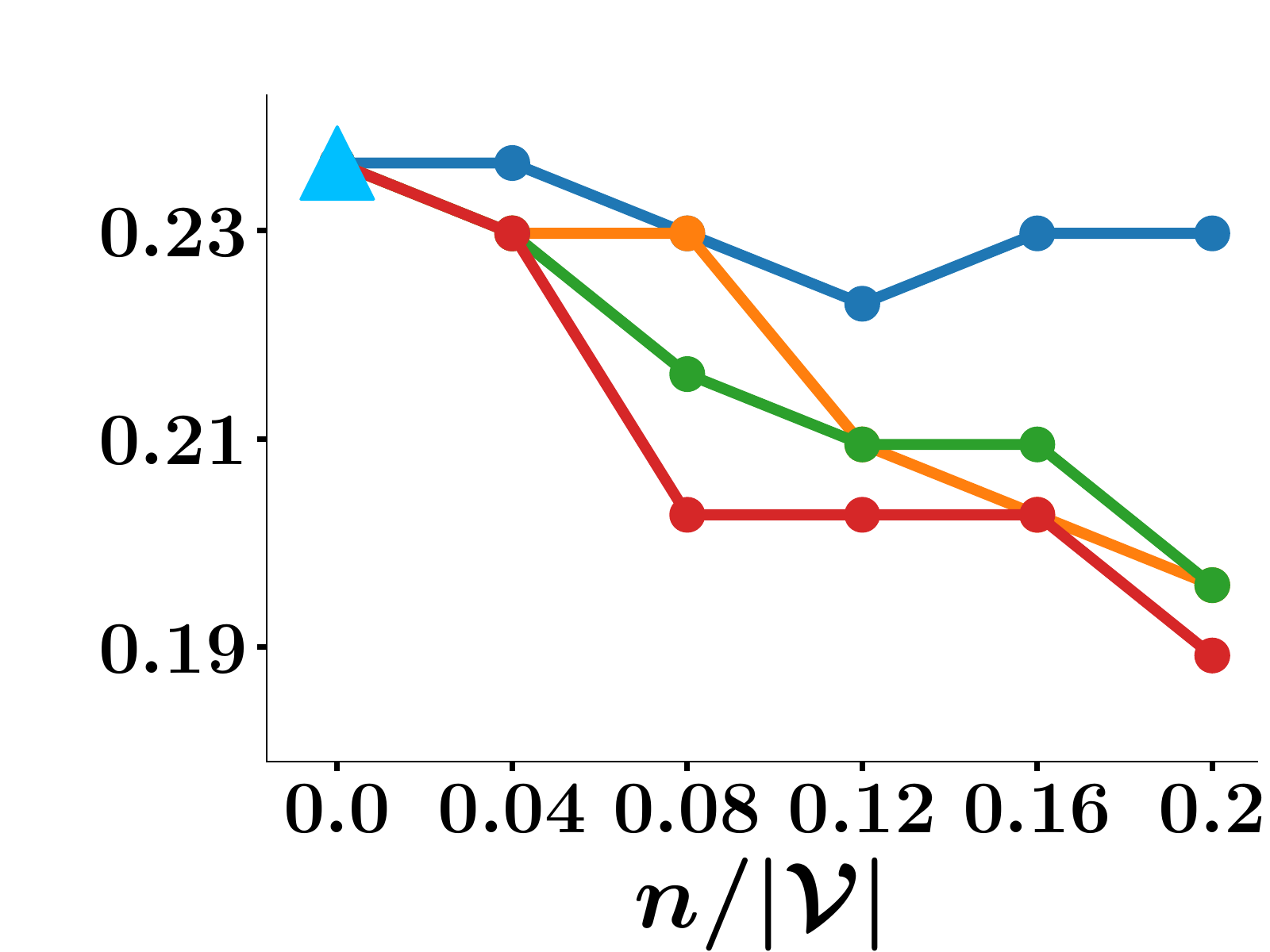}}\hspace{6mm}
\subfloat{\includegraphics[width=0.26\textwidth]{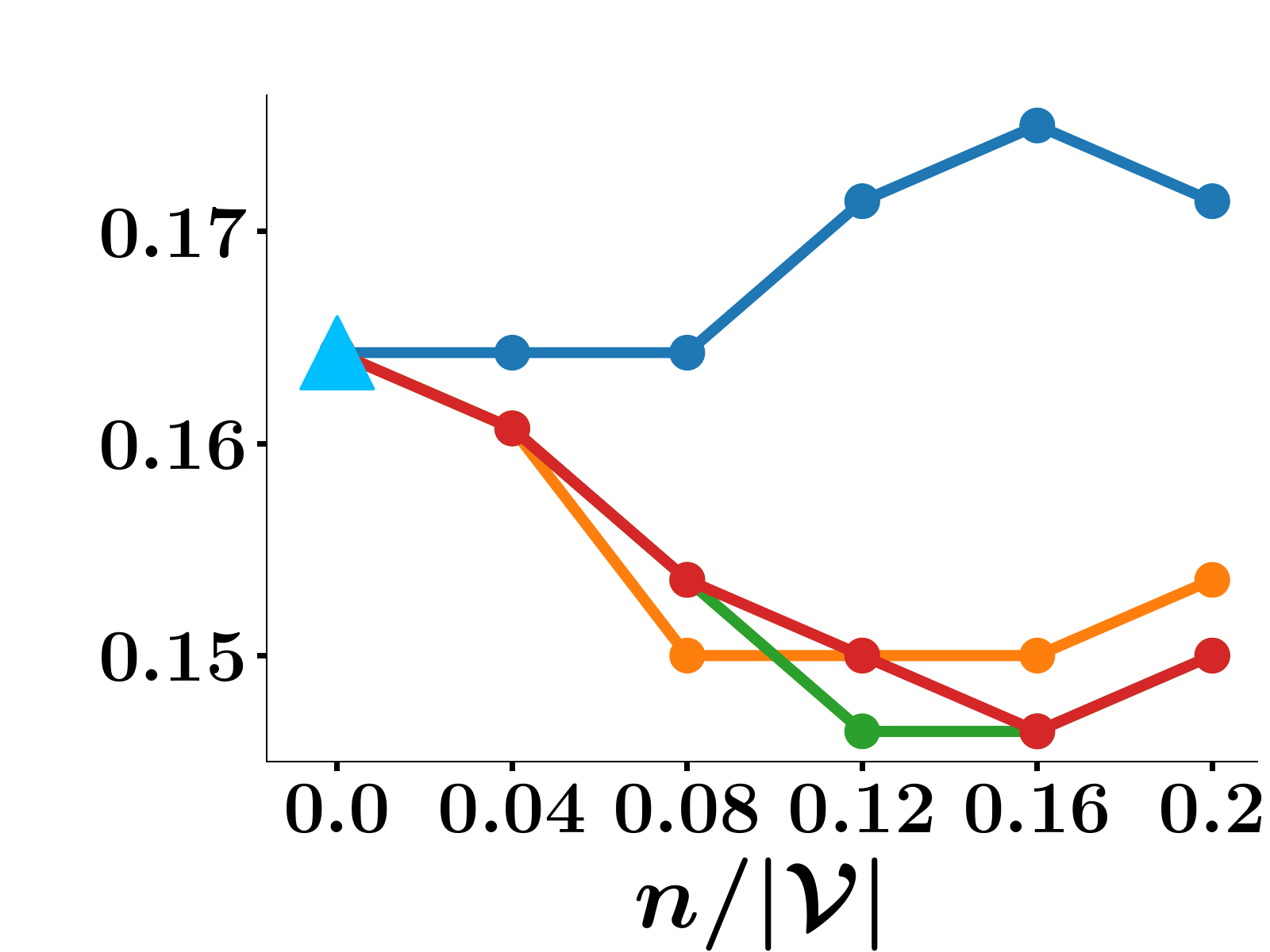}} \\[-1ex]
\subfloat[Messidor]{\setcounter{subfigure}{1} \includegraphics[width=0.26\textwidth]{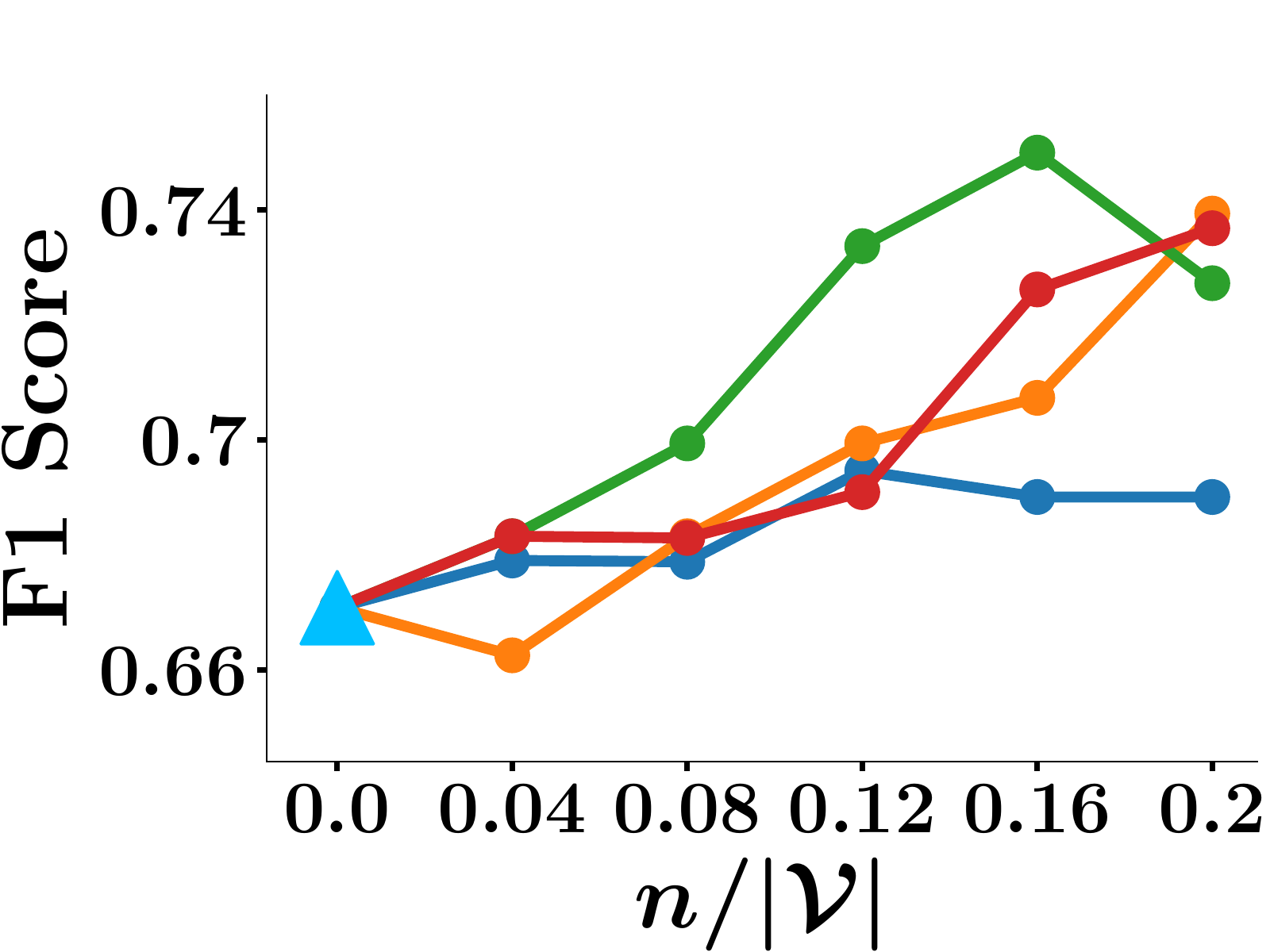}}\hspace{6mm}
\subfloat[Stare]{\includegraphics[width=0.26\textwidth]{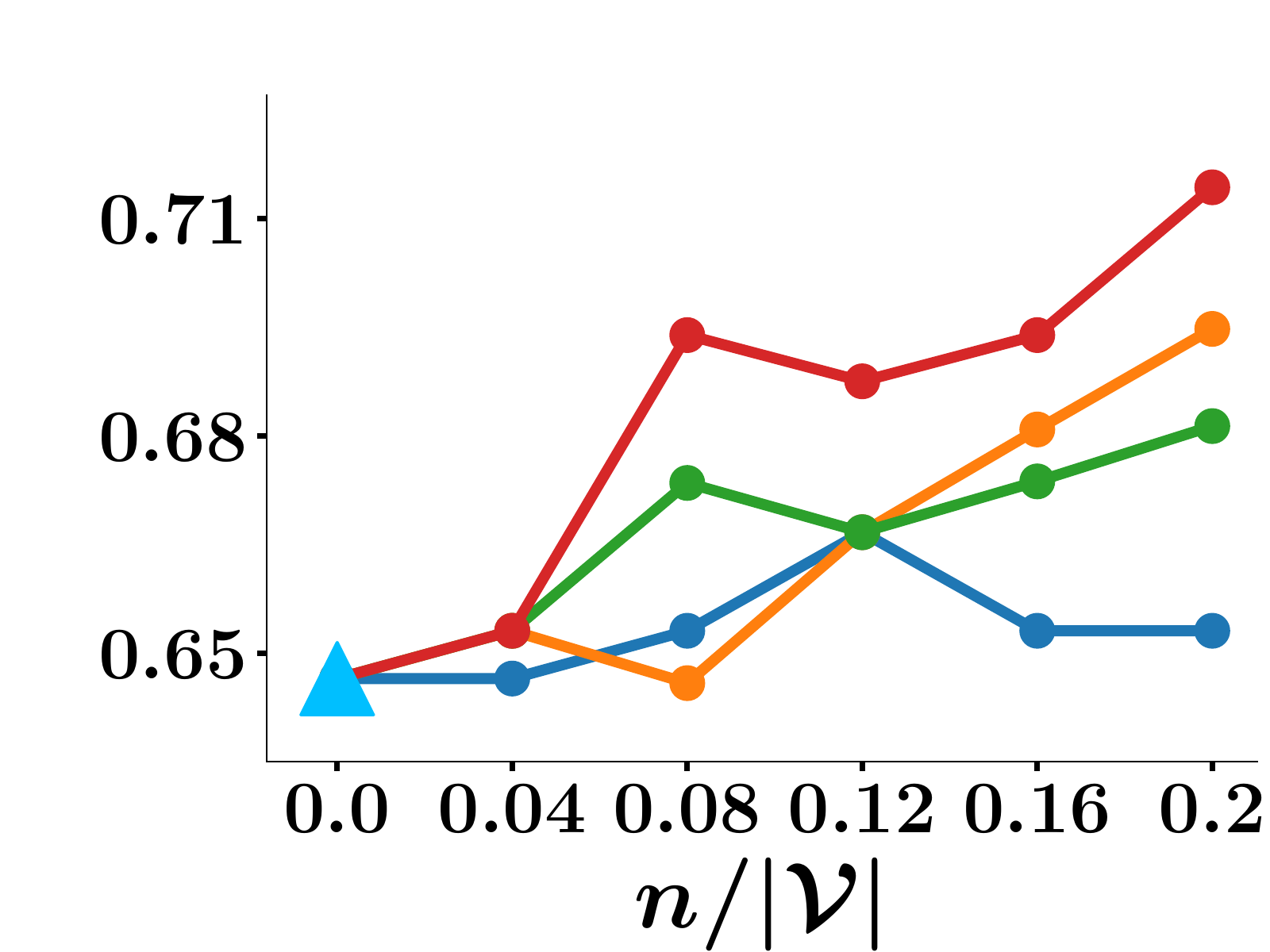}}\hspace{6mm}
\subfloat[Aptos]{\includegraphics[width=0.26\textwidth]{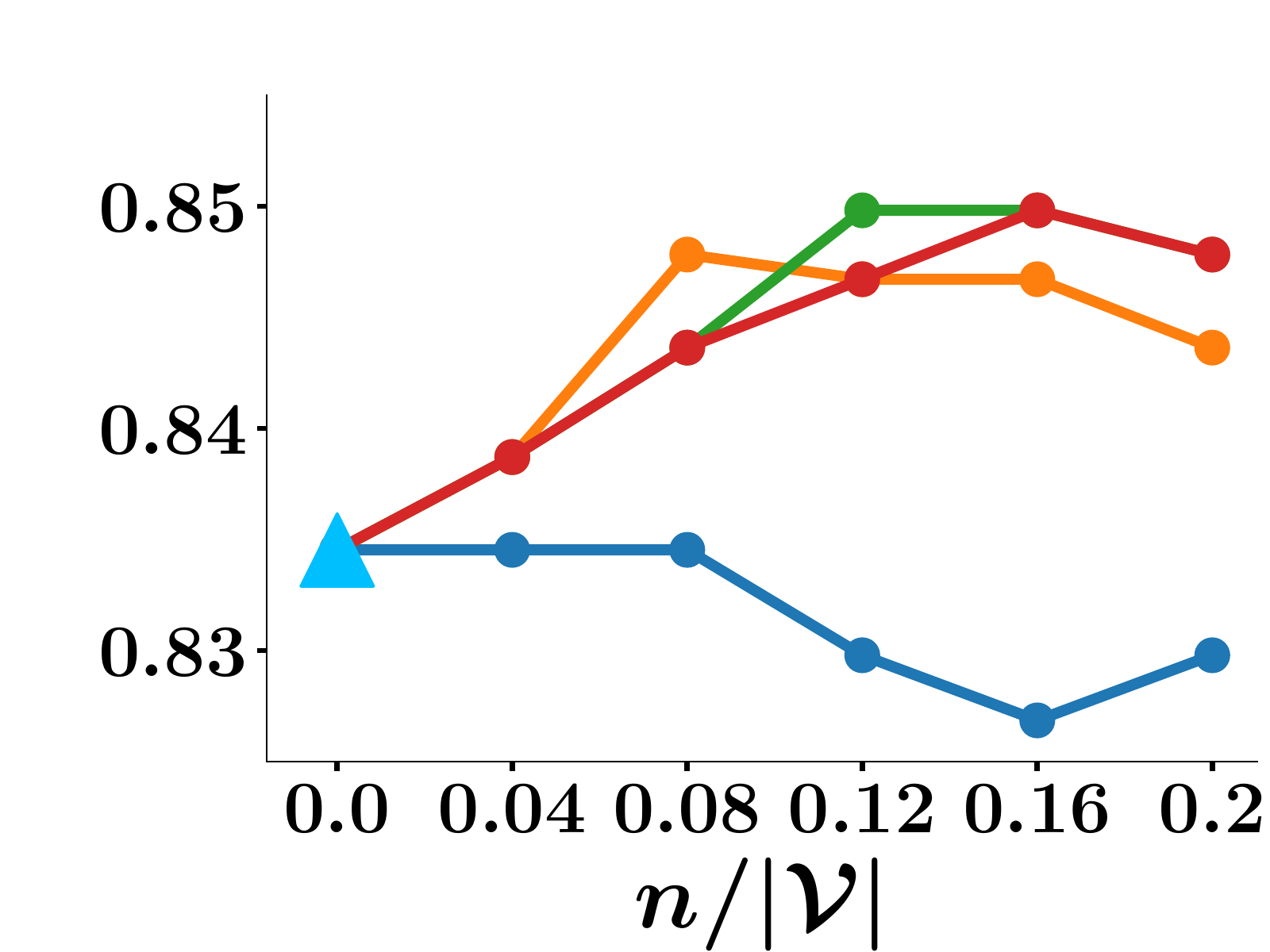}}  
\caption{Performance of all methods on three medical datasets.
%
If humans predict all testing samples (No automation), we have that $\PP(\hat{y}\neq y) = 0.15, 0.14, 0.20$ and
 $\text{F1-score} = 0.87, 0.78, 0.8$ for Messidor, Stare and Aptos datasets respectively. 
}

\label{fig:real-main}
\end{figure*}
\oursub{Experimental setup}
%
We experiment with three publicly available datasets~\cite{decenciere_feedback_2014,hoover2000locating}, each
of them from a diffe\-rent application in medical diagnosis:
\begin{itemize}
 \item[(i)]\textbf{Messidor:} It consists of $|\Vcal|=400$ eye images. 
Each image is assigned a score by one single medical expert, 
on a four-point scale, which measures the severity of a retinopathy.
\item[(ii)]\textbf{Stare:} It consists of $|\Vcal|=373$ retinal images. Each image is assigned a score by one single medical expert, 
on a five-point scale, which measures the severity of a retinal hemorrhage.
\item[(iii)]\textbf{Aptos: } 
It consists of $|\Vcal|=705$ retinal images. Each image is given a score by one single clinician, on a five-point scale, which measures the 
severity of diabetic retinopathy. 
\end{itemize}
%
%
Following previous work, we first generate a $1000$ dimensional feature vector using Resnet~\cite{resnet} for each sample in 
the Stare dataset and a $4096$-dimensional feature vector using VGG16~\cite{simonyan2014very} for each 
sample in the Messidor and the Aptos datasets.
%
%
Then, we use the top $50$ features, as identified by PCA, as $\xb$ in our experiments.
%
%
Moreover, for each sample, we set $y=-1$ if its severity score $q$ corresponds to one of the two lowest grades of the associated disease 
and $y=+1$ otherwise. 
%

For each sample, we only have access to the score $q$ given by a single human expert.
%
%
%
Therefore, we sample the scores given by humans from a categorical distribution $h(\xb) = \hat{q} \sim \text{Cat}(\pb_{\xb,q})$, 
where $\pb_{\xb,q} = \text{Dirichlet}(\chib_{\xb,q})$ are the probabilities of each potential score $\hat{q}$ for a sample $(\xb,y)$ 
and $\chib_{\xb,q}$ is a vector parameter that controls the human accuracy, and then scale these scores so that $h(\cdot)\in[-1,1]$
and compute the human error as $c(\xb,y)=\EE\left[(1-y h(\xb))_{+} \right]$.
For each sample $(\xb, y)$, the element of $\chib_{\xb,q}$ corresponding to the score $\hat{q} = q$ given by the human expert has the highest 
value.
%

Finally, for each dataset, we use 60\% of samples for training and 40\% of samples for testing, set the value of $\lambda$ using cross validation under full
automation, and followed the procedure described in Section~\ref{sec:formulation} to train a logistic regression model $\pi(d \,|\, \xb)$ that decides which 
samples to outsource to humans at test time. Here, the logistic regression model uses just $|\wb^*(\Vcal\cp\Scal)\phi(\xb)+b^*(\Vcal\cp\Scal)|$ 
as the only single feature.
Appendix~\ref{app:impl-real} contains additional details about
the generation of human scores,
the additional model $\pi$ and
the baseline algorithms.

\oursub{Result}
First, we compare the performance of all methods in Figure~\ref{fig:real-main}. The results show that the distorted greedy algorithm as well as its randomized variant:
%
%
(i) outperform the triage baselines in most of the automation levels and
(ii) benefit from human assistance both in cases when humans on their own (No automation) are better than machines on their own (Full automation) 
and vice versa (Panels (a) and (b) vs Panel (c));
and, (iii) perform comparably, sometimes one beating the other and vice versa (refer to Section 3.4. in~\citet{harshaw2019submodular} to better understand the reasons). 
Here, note that we intentionally experimented with different values of $\chib_{\bullet,\bullet}$ across different 
datasets so that the performance of the no automation baseline varies. Moreover, since in our datasets, the parameter 
$\rho^{*} = {\min\{|\Vcal^+|,|\Vcal^-|\}}/{|\Vcal|} \in (0.33,0.45)$, then to satisfy the conditions of Theorem~\ref{thm:soft-linear-subm}, we only 
experiment with $n/|\Vcal|\le 0.2< \rho^*$.

Next, we assess the influence that the human error $c(\xb,y)$ has on the performance of the greedy algorithm. 
Figure~\ref{fig:human-error-variation-real} summarizes the results for the Aptos dataset, which show that, as long as the human error is small, the performance of the distorted 
greedy algorithm improves with respect to the maximum number of samples $n$. However, for higher levels of human error, outsourcing samples does not help.
\begin{figure}[t]
\centering
\hspace{0.0cm}{\includegraphics[width=0.7\textwidth]{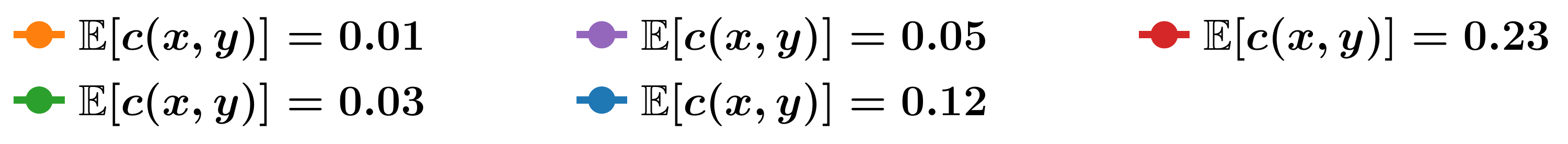}}\\[-2ex]
\subfloat[Missclassification error]{\includegraphics[width=0.27\textwidth]{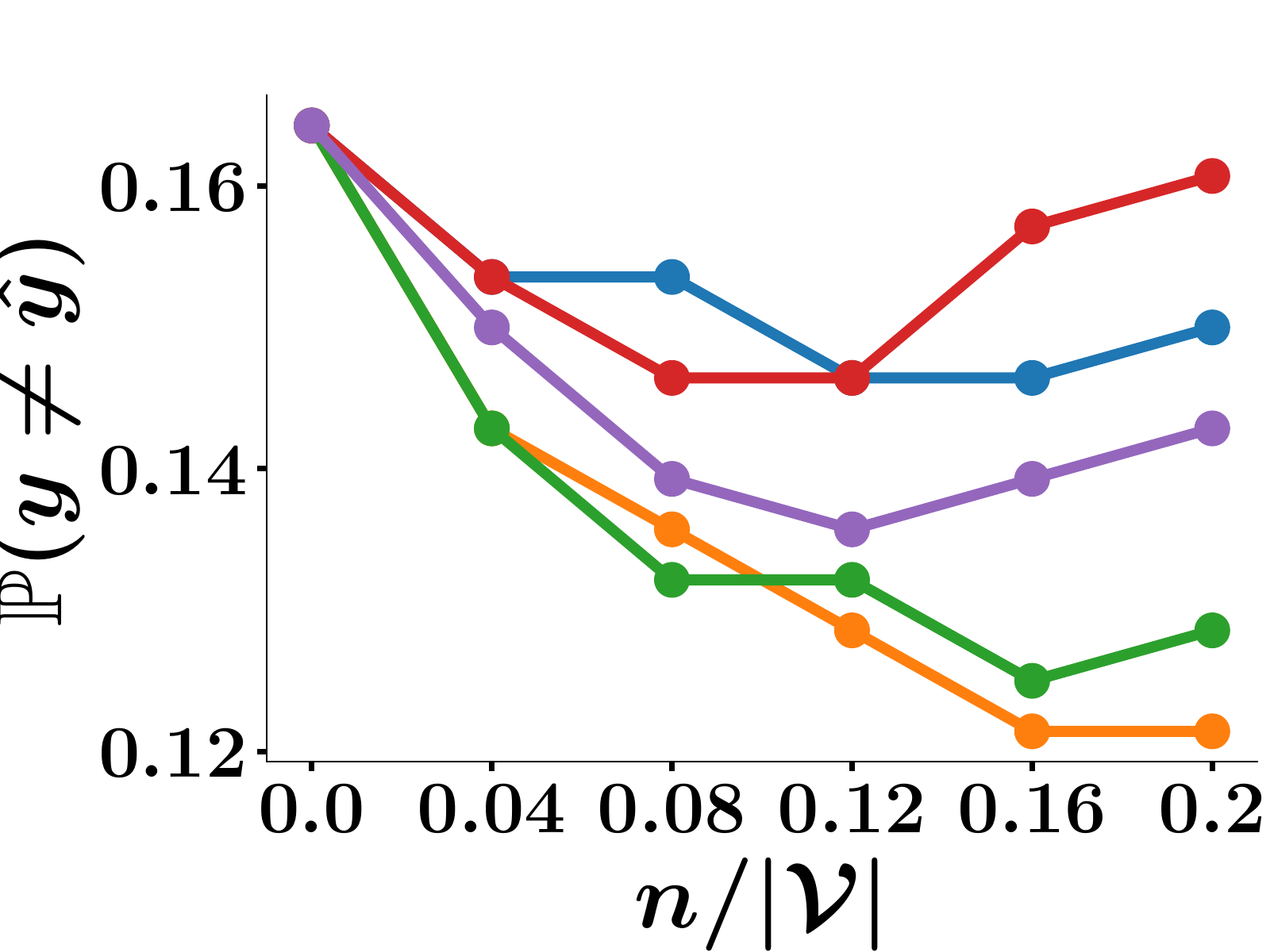}} \hspace*{15mm}
\subfloat[F1 score]{\includegraphics[width=0.27\textwidth]{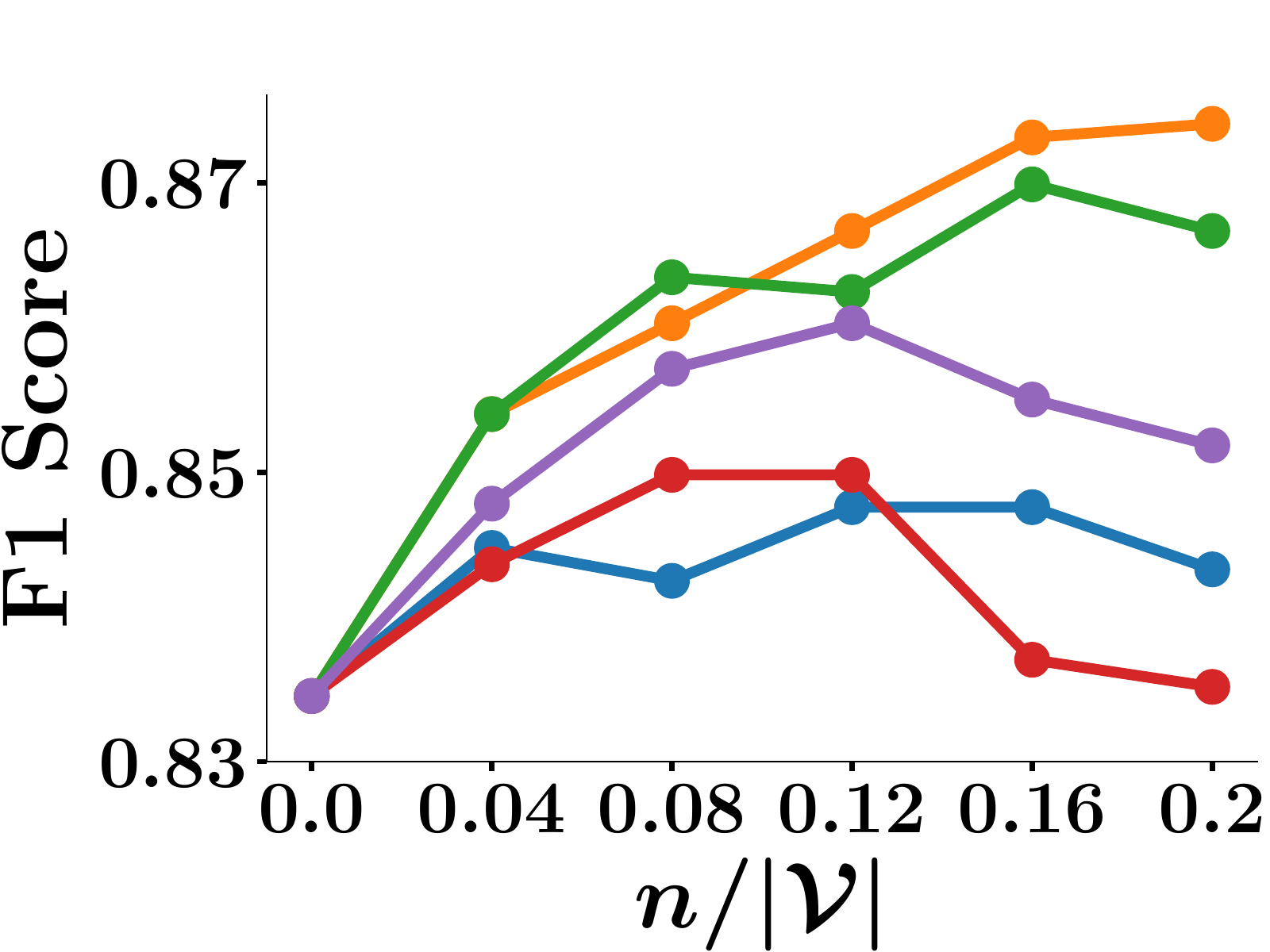}} 
\caption{
Performance of the stochastic distorted greedy algorithm against the maximum number of outsourced samples $n$ 
for different levels of human error $c(\xb,y)$ on the Aptos dataset.}
\label{fig:human-error-variation-real}
\end{figure}

\section{Conclusions}
\label{sec:conclusions}
In this paper, we have shown that, for support vector machines, we can solve the problem of classi\-fi\-ca\-tion under human assistance using algorithms
with approximation gua\-ran\-tees. 
Moreover, we have further shown that, under human assistance, support vector machines trained to operate under different automation levels can provide
superior empirical performance than those trained for full automation.

Our work also opens many interesting venues for future work.  For example, we have assumed that the human error is known, however, in practice, 
the spectrum of human abilities spans a broad range. It would be interesting to develop algorithms that, over time, adapt to the particular human(s) they are dealing with. 
Moreover, we have assumed that the human annotations are independent, however, there exist scenarios in which annotations are correlated, \eg, when a single user 
sequentially reviews a set of items in a session.
It would be interesting to lift the independence assumption and design algorithms that account for the correlation between annotations.
It would also be va\-lua\-ble to find tighter lower bounds on the parameter $\gamma$,  which better characterize the good empirical performance.
%
Finally, it would be worth to extend our analysis to other convex margin-based classifiers as well as design machine learning algorithms operating under different automation levels in sequential decision-making tasks, \eg,  semi-autonomous driving~\cite{meresht2020learning}.

\section*{Acknowledgements} 
This project has received funding from the European Research Council (ERC) under the European
Union'{}s Horizon 2020 research and innovation programme (grant agreement No. 945719). 

\bibliographystyle{plainnat}
\bibliography{refs}

\clearpage
\newpage

\begin{appendix}
\label{sec:appendix}
\section{Proof of Theorem~\ref{thm:np-hard}} \label{app:np-hard}
We prove that a particular instance of our problem is NP-hard.
Let $\ell(h_{w}(\xb_i), y_i) = (1-y \, \wb^\top \xb_i)^2$ and assume the human error $c(\xb_i,y_i)=0$ for all $i \in \Vcal$, 
$\lambda=0$, 
$d> |\Vcal|$, and 
$\Xb=[\xb^\top _1;\xb^\top _2;..;\xb^\top _{|\Vcal|}]$ has full row rank $|\Vcal|$.
Then, we can map our problem to the NP-hard problem of robust least square regression~\cite{bhatia2017consistent},
which is given by:
\begin{equation*}
\minz{\Scal,\wb, |\Scal|>n} \sum_{i\in\Scal}(r_i-\wb^\top \xb_i)^2,
\end{equation*}
where $r_i\in\RR$. 
More specifically, note that since $\Xb$ has full row rank then, it has a right inverse. Then, we define $\wb_0 = \Xb^{-1}_{R}(\yb-\rb)$, 
where $\Xb^{-1} _R$ is the right inverse of $\Xb$. Therefore, $\wb_0 ^\top \xb_i = \xb^\top _i \wb_0 = \left[\Xb \Xb^{-1} _{R}\right]_{i\star} (\yb-\rb) = y_i-r_i $. 
Finally, if we define a new variable $\wb'=\wb-\wb_0$, we can rewrite the optimization problem in Eq.~\ref{eq:optimization-problem} 
as follows:
\begin{align*}
\minz{\wb, \Vcal \backslash \Scal, |\Vcal \backslash \Scal|>|\Vcal|-n}  (1-y_i \wb^\top \xb_i )^2 &= \minz{\wb', \Vcal \backslash \Scal, |\Vcal \backslash \Scal|>|\Vcal|-n}  (1-y_i \wb'^\top \xb_i -y_i  \wb_0 ^\top \xb_i )^2\nonumber\\
 &= \minz{\wb', \Vcal \backslash \Scal, |\Vcal \backslash \Scal|>|\Vcal|-n}  (1-y_i\wb'^\top \xb_i -  y_i ^2 +y_i r_i )^2\nonumber\\
 &=\minz{\wb', \Vcal \backslash \Scal, |\Vcal \backslash \Scal|>|\Vcal|-n}  (y_i r_i-y_i\wb'^\top \xb_i )^2\\
 &=\minz{\wb', \Vcal \backslash \Scal, |\Vcal \backslash \Scal|>|\Vcal|-n}  ( r_i-\wb'^\top \xb_i )^2.
\end{align*}
This concludes the proof.

\section{Proof of Proposition~\ref{thm:mont}} \label{app:mont}
%
Without loss of generality, assume $b = 0$. Then, we have:
\begin{align*}
 g(\Scal   \cup \{j\})  - g(\Scal) &= \sum_{i\in\Vcal \backslash \Scal} [\lambda \bnm{\wb^{*}(\Vcal \backslash \Scal)}^2 + (1-y_i \wb^{*}(\Vcal \backslash \Scal)^{\top} \Phi(\xb_i))_+] \\
&- \sum_{i\in{\Vcal \backslash (\Scal\cup \{j\})}} [ \lambda \bnm{\wb^{*}(\Vcal \backslash (\Scal\cup \{j\}))}^2 + (1-y_i \wb^{*}(\Vcal \backslash (\Scal \cup \{j\})^{\top} \Phi(\xb_i))_+ ] \\
 & =  \lambda |\Vcal \backslash \Scal| \bnm{\wb^{*}(\Vcal \backslash \Scal)}^2  + \sum_{i\in\Vcal \backslash \Scal} (1-y_i \wb^{*}(\Vcal \backslash \Scal)^{\top} \Phi(\xb_i))_+ \\
  &-  \lambda (|\Vcal \backslash \Scal|-1) \bnm{\wb^{*}(\Vcal \backslash \Scal)}^2 - \sum_{\Vcal \backslash (\Scal\cup \{j\})} (1-y_i \wb^{*}(\Vcal \backslash \Scal)^{\top} \Phi(\xb_i))_+\\
 &+  \lambda (|\Vcal \backslash \Scal|-1) \bnm{\wb^{*}(\Vcal \backslash \Scal)}^2  + \sum_{i\in \Vcal \backslash (\Scal\cup \{j\}) } (1-y_i \wb^{*}(\Vcal \backslash \Scal)^{\top} \Phi(\xb_i))_+ \\
&- \lambda (|\Vcal \backslash \Scal|-1) \bnm{\wb^{*}(\Vcal \backslash (\Scal \cup \{j\}))}^2 - \sum_{i\in{\Vcal \backslash (\Scal\cup \{j\})}} (1-y_i \wb^{*}(\Vcal \backslash (\Scal \cup \{j\})^{\top} \Phi(\xb_i))_+ \\
&= \lambda  \bnm{\wb^{*}(\Vcal \backslash \Scal)}^2  +  [1-y_j \wb^{*}(\Vcal \backslash \Scal)^{\top} \Phi(\xb_j)]_+ \\
&+  \lambda (|\Vcal \backslash \Scal|-1) \bnm{\wb^{*}(\Vcal \backslash \Scal)}^2  + \sum_{i\in \Vcal \backslash (\Scal\cup \{j\}) } (1-y_i \wb^{*}(\Vcal \backslash \Scal)^{\top} \Phi(\xb_i))_+ \\
 & - \lambda (|\Vcal \backslash \Scal|-1) \bnm{\wb^{*}(\Vcal \backslash (\Scal \cup \{j\}))}^2 - \sum_{i\in{\Vcal \backslash (\Scal\cup \{j\})}} (1-y_i \wb^{*}(\Vcal \backslash (\Scal \cup \{j\})^{\top} \Phi(\xb_i))_+ \\
&\overset{(i)}{\ge} \lambda  \bnm{\wb^{*}(\Vcal \backslash \Scal)}^2  +  (1-y_j \wb^{*}(\Vcal \backslash \Scal)^{\top} \Phi(\xb_j))_+ \ge  0,
\end{align*}
where $\wb^{*}(\Vcal \backslash \Scal) = \argmin_{w} \sum_{i\in\Vcal \backslash \Scal} [\lambda \|\wb\|^2 + (1-y_i w^{\top} \Phi(\xb_i))_+]$ and
the inequality $(i)$ holds because 
\begin{align*}
\min_{\wb} & \sum_{i\in{\Vcal \backslash (\Scal\cup \{j\})}}  [ \lambda \|\wb\|^2 + (1-y_i \wb^{*}(\Vcal \backslash \Scal)^{\top} \Phi(\xb_i))_+ ]\nn\\
&= \sum_{i\in{\Vcal \backslash (\Scal\cup \{j\})}} [\lambda \|\wb^{*}(\Vcal \backslash (\Scal \cup \{j\})) \|^2 
 + (1-y_i \wb^{*}(\Vcal \backslash (\Scal \cup \{j\}))^{\top} \Phi(\xb_i))_+] \\
& \leq \sum_{i\in{\Vcal \backslash (\Scal\cup \{j\})}} [\lambda \|\wb^{*}(\Vcal \backslash \Scal) \|^2 
+ (1-y_i \wb^{*}(\Vcal \backslash \Scal)^{\top} \Phi(\xb_i))_+].
\end{align*}
%

%
%

\section{Proof of Theorem~\ref{thm:soft-linear-subm}} \label{app:thm:soft-linear-subm}
First, we find a lower bound for $g(\Scal \cup \{j\}) - g(\Scal)$ for every $j \notin \Scal$. Starting from Eq.~\ref{eq:soft-lin-svm-2} with $\Phi(\xb) = \xb$, we have:
\begin{equation*}
g(\Scal\cup \{j\}) - g(\Scal) \ge \lambda \bnm{\wb^{*}(\Vcal \backslash \Scal)}^2 + (1 - y_j(\wb^{*}(\Vcal \backslash \Scal)^{\top} \xb_j + b^{*}(\Scal)))_{+} \ge \lambda \bnm{\wb^{*}(\Vcal \backslash \Scal)}^2. \label{eq:fnumsoftlinear}
\end{equation*}
Then, to bound $\bnm{\wb^{*}(\Vcal \backslash \Scal)}^2$, we resort to Lemma~\ref{prop:svm-dual}, which equip us with the dual formulation of 
our problem, \ie,
\begin{equation*}
\wb^{*}(\Vcal \backslash \Scal)= {\dfrac{\sum_{i\in \Vcal \backslash \Scal}\alpha_{i}(\Vcal \backslash \Scal) y_i\xb_i}{2\lambda|\Vcal \backslash \Scal|}} \quad \text{and} \quad \sum_{i\in\Vcal \backslash \Scal}\alpha_{i}(\Vcal \backslash \Scal) y_i=0, \label{eq:coneq}
\end{equation*}
and the following condition:
\begin{align*}
  y_i(\wb^{*}(\Vcal \backslash \Scal)^{\top} \xb_i+b^{*}(\Vcal \backslash \Scal))< 1 \implies \alpha_{i}(\Vcal \backslash \Scal) =1,  
\end{align*}
as well as to the fact that, due to the upper bound on $|\Scal|$, $\Vcal \backslash \Scal$ contains at least one positive and one negative point since 
$|\Vcal \backslash \Scal| = |\Vcal|-|\Scal| > |\Vcal|-(|\Vcal^{+}|-s^*) = |\Vcal^{-}|+s^*$ and $|\Vcal \backslash \Scal|  > |V^{+}| + s^{*}$. 
More specifically, we consider the following two cases:

--- \emph{Case (a)}: There is at least one instance $i^+\in \Vcal \backslash \Scal$ with $y_{i^+}=+1$ and another instance $i^-\in \Vcal \backslash \Scal$ with $y_{i^-}=-1$ that satisfy    
$y_{i^+}(\wb^{*}(\Vcal \backslash \Scal)^{\top}\xb_{i^+}+b^{*}(\Vcal \backslash \Scal))\ge 1$ and $y_{i^-}(\wb^{*}(\Vcal \backslash \Scal)^{\top}\xb_{i^-}+b^{*}(\Vcal \backslash \Scal))\ge 1$, 
respectively. Then, by adding these two inequalities, we have the following bound: 
\begin{equation*}
 \wb^{*}(\Vcal \backslash \Scal)^{\top} (y_{i^+}\xb_{i^+}+y_{i^-}\xb_{i^-}) \ge 2 \implies \|\wb^{*}(\Vcal \backslash \Scal)\|^2 \ge \frac{4}{\|\xb_{i^+}-\xb_{i^-}\|^2} 
 \ge \frac{1}{\max_{i\in\Vcal}\bnm{\xb_i}^2}.
\end{equation*}

--- \emph{Case (b)}: If case (\emph{a}) does not happen, then it means that either all negative samples or all positive samples are misclassified. 
Without loss of generality, assume that all negative instances are miss-classified. Then, from the KKT conditions, we have that $\alpha_{i}(\Vcal \backslash \Scal)=1$
for all $i\in(\Vcal \backslash \Scal)\cap \Vcal^-$. Using this observation and the dual formulation of the problem above, we have:
\begin{equation*}
{\sum_{i\in(\Vcal \backslash \Scal)\cap \Vcal^+}\alpha_{i}(\Vcal \backslash \Scal)
= \sum_{i\in(\Vcal \backslash \Scal)\cap \Vcal^-}\alpha_{i}(\Vcal \backslash \Scal) =|(\Vcal \backslash \Scal)\cap\Vcal^-|}.
\end{equation*}
Then, using the above equality, we can derive the following bound:
\begin{align*}
\bnm{\wb^{*}(\Vcal \backslash \Scal)} & = \frac{1}{2\lambda|\Vcal \backslash \Scal|}
\bnm{ \sum_{i\in(\Vcal \backslash \Scal)\cap\Vcal^+}\alpha_{i}(\Vcal \backslash \Scal) \xb_i 
-  \sum_{i\in(\Vcal \backslash \Scal)\cap\Vcal^-}\alpha_{i}(\Vcal \backslash \Scal) \xb_i  } \\
&= \frac{1}{2\lambda|\Vcal \backslash \Scal|} 
 \bnm{ \dfrac{\sum_{i\in(\Vcal \backslash \Scal)\cap\Vcal^+}\alpha_{i}(\Vcal \backslash \Scal) \xb_i}{\sum_{i\in(\Vcal \backslash \Scal)\cap \Vcal^+}\alpha_{i}(\Vcal \backslash \Scal) }
 -  \dfrac{\sum_{i\in(\Vcal \backslash \Scal)\cap\Vcal^-}\alpha_{i}(\Vcal \backslash \Scal) \xb_i}{\sum_{i\in(\Vcal \backslash \Scal)\cap \Vcal^-}\alpha_{i}(\Vcal \backslash \Scal)}  }\cdot \sum_{i\in(\Vcal \backslash \Scal)\cap \Vcal^-}\alpha_{i}(\Vcal \backslash \Scal) 
\end{align*}

\begin{align*}
 & \overset{(i)}{\ge} \frac{1}{2\lambda|\Vcal \backslash \Scal|} \Delta_{1/|(\Vcal \backslash \Scal)\cap \Vcal^-|} \cdot |(\Vcal \backslash \Scal)\cap \Vcal^-|\nonumber \\
 & \overset{(ii)}{\ge} \frac{1}{2\lambda|\Vcal \backslash \Scal|}\Delta^* \cdot |(\Vcal \backslash \Scal)\cap \Vcal^-| \\
  & \overset{(iii)}{=} \frac{ \Delta^*   \sigma^* |\Vcal| }{2\lambda|\Vcal \backslash \Scal|}  \\
  & \ge \frac{ \Delta^*   \sigma^*   }{2\lambda },
\end{align*}
where
\begin{itemize}
\item the inequality $(i)$ is due to the fact that 
\begin{equation*}
\frac{\sum_{i\in(\Vcal \backslash \Scal)\cap\Vcal^\pm}\alpha_{i}(\Vcal \backslash \Scal) \xb_i}
{\sum_{i\in\Vcal \backslash \Scal\cap \Vcal^\pm}\alpha_{i}(\Vcal \backslash \Scal)}  \in \Ccal^{\pm} _{1/|(\Vcal \backslash \Scal)\cap \Vcal^-|}
\end{equation*}
and, by definition, it holds that $\bnm{\ab^+-\ab^-}>  
\Delta_{1/|(\Vcal \backslash \Scal)\cap \Vcal^-|}$ for $\ab^\pm \in\Ccal^{\pm} _{1/|(\Vcal \backslash \Scal)\cap \Vcal^-|}$;
\item the inequality $(ii)$ is due to $|(\Vcal \backslash \Scal)\cap \Vcal^-| \ge s^*$ (Fact~\ref{fact:scalBoundst}) which results in $\Delta_{1/|(\Vcal \backslash \Scal)\cap\Vcal^-|} \ge \Delta_{1/s^*}=\Delta^*$ 
 (Fact~\ref{fact:ccalprop1}); and,
\item the equality $(iii)$ is due to $|\Scal|<n< (\rho^*-\sigma^*)|\Vcal|$  and Fact~\ref{fact:ccalprop1}. 
\end{itemize}

Then, if we combine the bounds from both cases, we have that:
 \begin{equation*}
  g(\Scal\cup \{j\})-g(\Scal) \ge \lambda \bnm{\wb^{*}(\Vcal \backslash \Scal)}^2 \ge 
  \min\left\{ \dfrac{\left[\Delta^{\ast}\sigma^* \right]^2}{4\lambda}, \dfrac{\lambda}{\max_{i\in\Vcal}\bnm{\xb_i}^2}  \right\}.
 \end{equation*}
 
Second, we find a lower bound for $g(\Scal\cup \Lcal) - g(\Scal)$, with $\Lcal \subseteq \Vcal \backslash \Scal$.
To this end, we can directly use (I) of Lemma~\ref{lem:key} and obtain
\begin{align*}
 g(\Scal\cup \Lcal) -g(\Scal) 
  &\le
    \dfrac{\left(2\sqrt{\lambda}+\kappa\right) |\Lcal|}{\sqrt{\lambda}}  +   \dfrac{\left(2\sqrt{\lambda}+\kappa\right)^2 |\Lcal|}{2\lambda}
  \cdot \left( \dfrac{1}{2} 
 + \sqrt{\dfrac{1}{4} +\dfrac{4\lambda |\Vcal| 
 (1+\frac{\kappa}{\sqrt{\lambda}}) }{(2\sqrt{\lambda}+\kappa)^2 |\Lcal|^2}} \right) \\
&\qquad +|\Vcal|  (1+\frac{\kappa}{\sqrt{\lambda}}) \\ &=
    \eta |\Lcal| +   \dfrac{\eta^2}{2} |\Lcal|
  \cdot \left( \dfrac{1}{2} 
 + \sqrt{\dfrac{1}{4} +\dfrac{4 (\eta-1) |\Vcal|}{ \eta^2 |\Lcal|^2}} \right) +|\Vcal|  ( \eta-1),
\end{align*}
where the equality follows from the definition of $\eta$.
Finally, if we combine both lower bounds, we can bound the submodularity ratio from below as follows:
\begin{align}
 \dfrac{\sum_{j\in \Lcal} \left[ g(\Scal\cup \{j\})-g(\Scal) \right]}{g(\Scal \cup \Lcal)-g(\Scal)} 
 & \ge \dfrac{|\Lcal|\min\left\{ \dfrac{\left[\Delta^{\ast}\sigma^* \right]^2}{4\lambda}, \dfrac{\lambda}{\max_{i\in\Vcal}\bnm{\xb_i}^2}  \right\} }
 { \eta |\Lcal| +   \dfrac{\eta^2}{2} |\Lcal|
  \cdot \left( \dfrac{1}{2} 
 + \sqrt{\dfrac{1}{4} +\dfrac{4 (\eta-1) |\Vcal|}{ \eta^2 |\Lcal|^2}} \right) +|\Vcal|  ( \eta-1) }\nonumber\\
 & = \dfrac{\min\left\{ \dfrac{\left[\Delta^{\ast}\sigma^* \right]^2}{4\lambda}, \dfrac{\lambda}{\max_{i\in\Vcal}\bnm{\xb_i}^2}  \right\}}
 {\eta  +  \dfrac{\eta^2}{2} \left( \dfrac{1}{2} 
 + \sqrt{\dfrac{1}{4} +\dfrac{4 |\Vcal|   (\eta-1)}{\eta^2 |\Lcal|^2}} \right) +  \left(\eta-1\right) \dfrac{|\Vcal|}{|\Lcal|} }\nonumber\\
& \ge \dfrac{\min\left\{ \dfrac{\left[\Delta^{\ast}\sigma^* \right]^2}{4\lambda},  \dfrac{1}{(\eta-2)^2} \right\}}
 {\eta  +  \dfrac{\eta^2}{2} \left( \dfrac{1}{2} 
 + \sqrt{\dfrac{1}{4} +\dfrac{4 |\Vcal|   (\eta-1)}{\eta^2}} \right) +  \left(\eta-1\right) |\Vcal| }. \nonumber
\end{align}
 
\section{Proof of Theorem~\ref{thm:soft-nonlinear-subm}}  \label{app:soft-nonlinear-subm}
First, we find a lower bound for $g(\Scal \cup \{j\}) - g(\Scal)$ for every $j \notin \Scal$. Starting from Eq.~\ref{eq:soft-lin-svm-2}, we have that:
\begin{align*}
g(\Scal\cup \{j\}) - g(\Scal) &\ge \lambda \bnm{\wb^{*}(\Vcal \backslash \Scal)}^2 + (1 - y_j(\wb^{*}(\Vcal \backslash \Scal)^{\top} \Phi(\xb_j) + b^{*}(\Scal)))_{+} \\
&\ge \lambda \bnm{\wb^{*}(\Vcal \backslash \Scal)}^2. \label{eq:fnumsoftlinear}
\end{align*}
Then, to bound $\bnm{\wb^{*}(\Vcal \backslash \Scal)}^2$, we resort to Lemma~\ref{prop:svm-dual}, which equip us with the dual formulation of 
our problem, \ie,
\begin{equation*}
\wb^{*}(\Vcal \backslash \Scal)^{\top} {\Phi(\xb)} = {\dfrac{\sum_{i\in \Vcal \backslash \Scal}\alpha_{i}(\Vcal \backslash \Scal) y_i\ K(\xb, \xb_i)}{2\lambda|\Vcal \backslash \Scal|}}, \, \sum_{i\in\Vcal \backslash \Scal}\alpha_{i}(\Vcal \backslash \Scal) y_i=0, \, \text{and} \, \alpha_i(\Vcal \backslash \Scal)=0 \ \forall i\in\Scal,
\end{equation*}
and the KKT conditions, \ie,
\begin{align*}
 y_i(\wb^{*}(\Vcal \backslash \Scal)^{\top} \Phi(\xb_i)+b^{*}(\Vcal \backslash \Scal))< 1 \implies  \alpha_{i}(\Vcal \backslash \Scal) =1   
\end{align*}
as well as the fact that, due to the upper bound on $|\Scal|$, $\Vcal \backslash \Scal$ contains at least one positive and one negative points since 
$|\Vcal \backslash \Scal| =|\Vcal|-|\Scal|> |\Vcal|-(|\Vcal^+|-\sigma^*|\Vcal|) = |\Vcal^{-}|+\sigma^*|\Vcal|$ and $|\Vcal \backslash \Scal| > |\Vcal^{+}|+\sigma^*|\Vcal|$. 
More specifically, we consider the following two cases:

--- \emph{Case (a)}: There is at least one instance $i^+\in \Vcal \backslash \Scal$ with $y_{i^+}=+1$ and an instance $i^-\in \Vcal \backslash \Scal$ with $y_{i^-}=-1$, which satisfy    
$y_{i^+}( \wb^{*}(\Vcal \backslash \Scal)^{\top} \Phi(\xb_{i^+}) + b^{*}(\Vcal \backslash \Scal))\ge 1$ and $y_{i^-}(\wb^{*}(\Vcal \backslash \Scal)^{\top} \Phi(\xb_{i^-}) +b^{*}(\Vcal \backslash \Scal))\ge 1$, respectively. 
Then, by adding these two inequalities, we have the following bound: 
\begin{equation*}
\wb^{*}(\Vcal \backslash \Scal)^{\top} (y_{i^+} \Phi(\xb_{i^+}) +y_{i^-} \Phi(\xb_{i^-}))
\ge 2 \implies \bnm{\wb^{*}(\Vcal \backslash \Scal)}  >\frac{1}{\max_{i\in\Vcal} \sqrt{K(\xb_i,\xb_i)}}.
\end{equation*}

--- \emph{Case (b)}: If case (\emph{a}) does not happen, then it means that either all negative samples or all positive samples are miss-classified. 
Without loss of generality, assume that all negative instances are miss-classified. 
Then, from the KKT conditions, we have that $\alpha_{i}(\Vcal \backslash \Scal)=1,\ $ for $i\in(\Vcal \backslash \Scal)\cap \Vcal^-$.
Using this observation and the dual formulation of the problem above, we have that: 
\begin{equation*}
{\sum_{i\in(\Vcal \backslash \Scal)\cap \Vcal^+}\alpha_{i}(\Vcal \backslash \Scal) = 
\sum_{i\in(\Vcal \backslash \Scal)\cap \Vcal^-}\alpha_{i}(\Vcal \backslash \Scal) =|(\Vcal \backslash \Scal)\cap\Vcal^-|}.
\end{equation*}
Then, using the above equality, we can derive the following bound:
\begin{align*}
\bnm{\wb^{*}(\Vcal \backslash \Scal)}^2 &= \dfrac{\alphab(\Vcal \backslash\Scal)^{\top} 
\Yb^\top \KK \Yb \alphab(\Vcal \backslash\Scal)}{4\lambda^2 |\Vcal \backslash \Scal|^2} \nn\\
&=  \dfrac{\mub(\Vcal \backslash \Scal)^{\top} \Yb^\top \KK \Yb \mub(\Vcal \backslash \Scal)}{4\lambda^2 |\Vcal \backslash \Scal|^2}  
\cdot \left(\sum_{i\in(\Vcal \backslash \Scal)\cap \Vcal^-}\alpha_{i}(\Vcal \backslash \Scal)\right)^2
\\
& \ge \dfrac{\zeta |(\Vcal \backslash \Scal)\cap \Vcal^-|^2}{{4\lambda^2 |\Vcal \backslash \Scal|^2}} \\
& \overset{(i)}{\ge} \dfrac{\zeta \sigma^{*2}|\Vcal|^2}{{4\lambda^2 |\Vcal \backslash \Scal|^2}}
\ge \dfrac{\zeta \sigma^{*2}}{4\lambda^2},
\end{align*}
where $\mu(\Vcal \backslash \Scal)_i= \dfrac{\alpha_{i}(\Vcal \backslash \Scal)}{\sum_{i\in (\Vcal \backslash \Scal)\cap \Vcal^-} \alpha_{i}(\Vcal \backslash \Scal)}$ and the inequality $(ii)$ is due 
to Fact~\ref{fact:scalBoundst}.

Finally, if we combine both lower bounds, proceed similarly as in the proof of Theorem~\ref{thm:soft-linear-subm} in Appendix~\ref{app:thm:soft-linear-subm}, and apply Lemma~\ref{lem:key}, we obtain 
the required lower bound on the submodularity ratio $\gamma$.
%

\section{Hard margin linear SVMs} \label{app:hard-margin-linear-svm}
For hard margin linear SVMs, we can rewrite the minimization problem defined in Eq.~\ref{eq:optimization-problem} as follows:
\begin{equation} \label{eq:hard-lin-svm} 
\begin{split}
\underset{\Scal, \wb, b}{\text{minimize}} & \quad |\Vcal \backslash \Scal| \, \lambda \|\wb\|^2 + \sum_{i \in \Scal} [1-y_i h(\xb_i)]_{+}\\
\text{subject to} & \quad | \Scal | \leq n \\
& \quad  y_i(\wb^\top \xb_i  +b)\ge 1 \quad \forall i\in \Vcal \backslash \Scal.
\end{split}
\end{equation}
%
%
For any given set $\Scal$, let $\wb^{*}(\Vcal \backslash \Scal)$ and $b^{*}(\Vcal \backslash \Scal)$ be the parameters that minimize the objective
function above subject to the constraints. Moreover, note that these parameters can be found in polynomial time since $|\Vcal \backslash \Scal| \lambda \|\wb\|^2$ is a convex 
function and $y_i(\wb^\top \xb_i +b)\ge 1$ is a convex constraint. 
Then, we can rewrite the above minimization problem as a set function maximization problem:
\begin{equation} \label{eq:hard-lin-svm-2} 
\underset{\Scal}{\text{maximize}} \quad \explainup{ \lambda \left[ |\Vcal| \|\wb^{*}(\Vcal)\|^2 - |\Vcal \backslash \Scal| \|\wb^{*}(\Vcal \backslash \Scal)\|^2 \right]}{g(\Scal)} - \explain{\sum_{i \in \Scal} [1-y_i h(\xb_i)]_{+}}{c(\Scal)} 
\quad \text{subject to} \quad | \Scal | \leq n.
\end{equation}
In the above, the constant term $|\Vcal| \|\wb^{*}(\Vcal)\|^2$ ensures that the function $g(\Scal)$ is non-negative, it readily follows from Proposition~\ref{thm:mont} that the function $g(\Scal)$ is monotone, and the function $c(\Scal)$ is clearly non-negative and modular.

Finally, the following proposition provides a lower bound on the submodularity ratio of $g(\Scal)$:
\begin{proposition}\label{thm:weak-hard-lin}
The set function $g(\Scal)$, defined in Eq.~\ref{eq:hard-lin-svm-2}, is $\gamma$-weakly submodular with $\gamma \geq {1}/{|\Vcal|}$.
\end{proposition}
\begin{proof}
From the second part of the proof in Proposition~\ref{thm:mont}, for $j\notin \Scal$ we have:
\begin{equation*}
g(\Scal\cup \{ j \}) - g(\Scal) \ge \lambda \|\wb^{*}(\Vcal \backslash \Scal)\|^2. 
\end{equation*}
Furthermore, for any $\Lcal \subseteq \Vcal$ such that $\Lcal \cap \Scal=\emptyset$, we have:
\begin{align*}
g(\Scal\cup \Lcal) - g(\Scal) &= |\Vcal \backslash \Scal| \, \lambda \| \wb^{*}(\Vcal \backslash \Scal) \|^2  - (|\Vcal \backslash \Scal|-|\Lcal|) \, \lambda  \| \wb^{*}(\Vcal \backslash (\Scal \cup \Lcal)) \|^2  \\
&\le |\Vcal \backslash \Scal| \, \lambda  \|\wb^{*}(\Vcal \backslash \Scal)\|^2.  
\end{align*}
From above two equations, and using the definition of weak submodularity we conclude the proof,
\begin{equation*}
  \dfrac{\sum_{j\in \Lcal}\left[ g(\Scal\cup \{ j \}) - g(\Scal) \right] }{g(\Scal\cup \Lcal) - g(\Scal) } \ge \dfrac{|\Lcal|}{|\Vcal \backslash \Scal|} \ge  \dfrac{1}{|\Vcal|}. 
\end{equation*}
\end{proof}
\section{Soft margin SVMs without offset} \label{app:soft-linear-without-offset-subm}
For a nonlinear soft margin SVMs without offset, we can rewrite the minimization problem defined in Eq.~\ref{eq:optimization-problem} as follows:
\begin{equation} \label{eq:soft-lin-svm-without-offset} 
\begin{split}
\underset{\Scal, \wb, b}{\text{minimize}} & \quad \sum_{i \in \Vcal \backslash \Scal} \explain{\left[ \lambda \|\wb\|^2 + (1-y_i \wb^\top\Phi(\xb_i) )_{+} \right]}{\ell(h_{\wb, b}(\xb_i), y_i)} + \sum_{i \in \Scal} \explain{[1-y_i h(\xb_i)]_{+}}{c(\xb_i, y_i)}\\
\text{subject to} & \quad | \Scal | \leq n.
\end{split}
\end{equation}
For any given set $\Scal$, let $w^{*}(\Vcal \backslash \Scal)$ be the parameter that minimizes the objective function above, \ie, $w^{*}(\Vcal \backslash \Scal) = \argmin_{w} \sum_{i\in\Vcal \backslash \Scal} [\lambda \|\wb\|^2 + (1-y_i \wb^{\top} \Phi(\xb_i))_+]$. 
Then, we can rewrite the above minimization problem as a set function maximization problem:
\begin{equation} \label{eq:soft-lin-svm-without-offset-2} 
\begin{split}
\underset{\Scal}{\text{maximize}} & \quad \explainup{ a(\Vcal) - |\Vcal \backslash \Scal | \, \lambda \|\wb^*(\Vcal \backslash \Scal)\|^2 - \sum_{i \in \Vcal \backslash \Scal} [1-y_i \wb^{*}(\Vcal \backslash \Scal)^\top \Phi(\xb_i) ]_{+}
}{g(\Scal)} \\
& \qquad \qquad - \explain{\sum_{i \in \Scal} [1-y_i h(\xb_i)]_{+}}{c(\Scal)}  \\
\quad \text{subject to} & \quad | \Scal | \leq n,
\end{split}
\end{equation}
where the constant $a(\Vcal) = |\Vcal| \, \lambda \|\wb^*(\Vcal)\|^2 + \sum_{i \in \Vcal} [1-y_i \wb^{*}(\Vcal)^\top \Phi(\xb_i) ]_{+}$ ensures that the function 
$g(\Scal)$ is non-negative. 

It readily follows from Proposition~\ref{thm:mont} that the function $g(\Scal)$ is monotone, and the function $c(\Scal)$ is clearly non-negative 
and modular. Moreover, if we define the distance $\Delta^{*}$ similarly as in Section~\ref{sec:algorithm}, we can characterize the submodularity ratio 
of the function $g(\Scal)$ in terms of the amount of overlap between feature vectors with positive and negative labels, as measured by $\Delta^{*}$, 
using the following Theorem:
\begin{theorem} \label{thm:soft-linear-without-offset-subm}
If $\eta= {\left(2\sqrt{\lambda}+\max_{i\in\Vcal} \sqrt{K(\xb_i, \xb_i)}\right) }/{\sqrt{\lambda}}$, then
the submodularity ratio of the function $g(\Scal)$ is given by
%
%
\begin{equation*}
\gamma\ge 
\gamma^*= 
  \dfrac{  \min\left\{   \dfrac{1}{(\eta-2)^2}, \dfrac{1}{2} \right\} }
 {  \eta  +   \dfrac{\eta^2}{2}   }.
 \end{equation*} 
\end{theorem}

\begin{proof}
Starting from Eq.~\ref{eq:soft-lin-svm-without-offset-2}, we have the following bound:
\begin{align*}
 g(\Scal\cup \{j\})-g(\Scal) & \ge \lambda \bnm{\wb^{*}(\Vcal \backslash \Scal)}^2 + (1-y_j \wb^{*}(\Vcal \backslash \Scal)^{\top} \Phi(\xb_j))_{+} \\
& \ge \min_{\wb}  \lambda \bnm{\wb}^2 + (1-y_j \wb^{\top} \Phi(\xb_j))_{+} \\
& \ge \max_{\alpha\in[0,1]} \alpha - \dfrac{\alpha^2 K(\xb_j, \xb_j)}{4\lambda} \quad \text{(Converting to the dual maximization problem)} \label{eq:singleDual} \\
& \ge \min\left\{\dfrac{\lambda}{\max_{i\in\Vcal} K(\xb_i, \xb_i)},\dfrac{1}{2}\right\} \quad \text{(Fact~\ref{fact:quadratic})}.
\end{align*}
Now, consider a set $\Lcal$ such that $\Scal \cap \Lcal =\emptyset$.  To bound $g(\Scal \cup \Lcal) -g(\Scal)$,
we directly use (II) in Lemma~\ref{lem:key}, with $\kappa=\max_{\xb} \sqrt{K(\xb,\xb)}$, which gives:
\begin{equation*}
 g(\Scal \cup \Lcal) - g(\Scal) 
  \le
    \dfrac{\left(2\sqrt{\lambda}+\kappa\right) |\Lcal|}{\sqrt{\lambda}}  +   \dfrac{\left(2\sqrt{\lambda}+\kappa\right)^2 |\Lcal|}{2\lambda}
=\eta |\Lcal| +   \dfrac{\eta^2}{2} |\Lcal|,
\end{equation*}
where the equality is obtained by plugging in the value of $\eta$.
Finally, we have:
\begin{align}
 \dfrac{\sum_{j\in \Lcal} [ g(\Scal\cup \{j\})-g(\Scal) ]}{g(\Scal \cup \Lcal)-g(\Scal)} 
 & \ge \dfrac{|\Lcal|  \min\left\{   \dfrac{1}{(\eta-2)^2}, \dfrac{1}{2} \right\} }
 {  \eta |\Lcal| +   \dfrac{\eta^2}{2} |\Lcal|  }
  =   \dfrac{  \min\left\{   \dfrac{1}{(\eta-2)^2}, \dfrac{1}{2} \right\} }
 {  \eta  +   \dfrac{\eta^2}{2}   }\nonumber .
 \end{align}
 \end{proof}
\vspace{-6mm}
 \section{Lemmas and Facts} \label{app:facts}
 
 \begin{lemma}\label{prop:svm-dual}
Given the set $\Scal$, the dual of the optimization problem: 

(i) For the soft margin linear SVM with offset is given by 
\begin{align} 
   \underset{\alphab}{\text{maximize}}& \quad
\alphab^\top \oneb_{|\Vcal|} - \frac{1}{4\lambda|\Vcal \backslash \Scal|}\sum_{i,j\in\Vcal}\alpha_i \alpha_j y_i y_j \xb^\top _i \xb _j \label{eq:xx2}\\
 \text{subject to} &\quad 0 \le \alpha_i   \le 1, \quad i \in\Vcal \nonumber \\
 &\quad \alpha_i=0,  \ \ \ \qquad i \in \Scal \nonumber \\
 & \quad \sum_{i\in \Vcal \backslash \Scal}\alpha_i y_i  =0. \label{eq:conlast}
\end{align}
Moreover, if $\alphab^{*}(\Vcal \backslash \Scal)$ be the optimal solution of the above problem, we have $\wb^{*}(\Vcal \backslash \Scal)=
 {\sum_{i\in \Vcal \backslash \Scal}\alpha_{i}(\Vcal \backslash \Scal) y_i\xb_i} / {2\lambda|\Vcal \backslash \Scal|}$.

\noindent (ii) For the soft margin linear SVM without offset, is the same as above optimization problem except the absence of the equality constraint in Eq.~\ref{eq:conlast}.

\noindent (iii) For the soft margin nonlinear SVM with offset is given by
\begin{align} 
   \underset{\alphab}{\text{maximize}}& \quad
\alphab^\top \oneb_{|\Vcal|} - \frac{1}{4\lambda|\Vcal \backslash \Scal|}\sum_{i,j\in\Vcal}\alpha_i \alpha_j y_i y_j K(\xb _i, \xb _j)  
 \label{eq:kxx2}\\
 \text{subject to} &\quad 0 \le \alpha_i   \le 1, \quad i \in\Vcal \nonumber \\
 &\quad \alpha_i=0,  \ \ \ \qquad i \in \Scal \nonumber \\
 & \quad \sum_{i\in \Vcal \backslash \Scal}\alpha_i y_i  =0, \label{eq:conlastx}
\end{align}
where $\oneb_{\Vcal}$ is a $|\Vcal|$ dimensional vector with unit entries. Moreover, if $\alphab^{*}(\Vcal \backslash \Scal)$ be the optimal solution of the above problem, we have $\wb^*(\Vcal\cp\Scal)(\xb)=
{\sum_{i\in \Vcal \backslash \Scal}\alpha_{i}(\Vcal \backslash \Scal) y_i K(\xb,\xb_i)} / {2\lambda|\Vcal \backslash \Scal|} $.

\noindent (iv) For the soft margin nonlinear SVM without offset, 
is same as the above optimization problem except the absence of the equality constraint in Eq.~\ref{eq:conlastx}.

\noindent (v) Finally, the dual variables and the positions of the corresponding training instances, for $i\in\Vcal \backslash \Scal$, relate in the following way:
\begin{align}
&     y_i \cdot h_{\wbvs,\bbvs} (\xb_i) > 1  \implies  \alpha_{i}(\Vcal \backslash \Scal)=0   \label{eq:a1} \\ 
&  y_i \cdot h_{\wbvs,\bbvs}(\xb_i) < 1    \implies \alpha_{i}(\Vcal \backslash \Scal)=1  \label{eq:a2} \\
& y_i \cdot h_{\wbvs,\bbvs} (\xb_i) = 1   \implies \alpha_{i}(\Vcal \backslash \Scal)\le 1  \label{eq:a3}
\end{align}
\end{lemma}
\begin{proof}
(\emph{i}) The coefficients of different terms in the
objective function is different from traditional objective function. 
The proof is almost identical to \cite{cristianini2000introduction}.
However, we provide the proof for self completion.
By introducing slack variables $\xi_j$, the optimization problem of soft linear SVM with offset is equivalent to:
\begin{align*}
 \minz{\wb,b, \xib}& \qquad \lambda |\Vcal \backslash \Scal| \| \wb \|^2 +\sum_{j\in \Vcal \backslash \Scal} \xi_j\\
 \text{subject to}& \qquad y_j(\wb^\top\xb_j+b) \ge 1-\xi_j, \ \  j\in\Vcal \backslash \Scal \\
 &\qquad \xi_j \ge 0, \qquad\qquad\qquad\quad\ \  j\in\Vcal \backslash \Scal.
\end{align*}
Solving the above optimization problem is equivalent to solving  its dual problem.
The Lagrange dual function $g$ is the infimum of Lagrangian, so we have:
\begin{align}\label{eq:ldef}
g(\alphab,\betab) =
\underset{\wb,b,\xib}{\min}
\sum_{j \in \Vcal \backslash \Scal}\left(\lambda \| \wb \|_2 ^2+\xi_j -\alpha_j (y_j(\wb^\top\xb_j+b) - 1 +\xi_j) -\beta_j\xi_j \right),
\end{align}
where $\xib=[\xi_j]_{j\in\Vcal \backslash \Scal}$, $\alphab=[\alpha_j]_{j\in\Vcal}$ and $\betab=[\beta_j]_{j\in\Vcal}$.
To find the above minimum, we differentiate the Lagrangian w.r.t. the primal variables to get
\begin{align}
\dfrac{\dd g}{\dd \wb} & =2\lambda |\Vcal \backslash \Scal| \wb- \sum_{j\in \Vcal \backslash \Scal}\alpha_j y_i\xb_j =0 \implies  \wb^{*}(\Vcal \backslash \Scal)= \frac{\sum_{j\in \Vcal \backslash \Scal}\alpha_j y_j\xb_j}{2\lambda|\Vcal \backslash \Scal|} 
\label{eq:duw}\\
\dfrac{\dd g}{\dd b}&=-\sum_{j\in \Vcal \backslash \Scal}\alpha_j y_j=0\implies \sum_{j\in \Vcal \backslash \Scal}\alpha_j y_j=0 \label{eq:ddb}\\
\dfrac{\dd g}{\dd \xi_j}&=0 \implies \alpha_j+\beta_j=1\quad \forall j \in  {\Vcal\backslash\Scal}. \label{eq:linalpha}
\end{align}
By putting these optimal values in Eq.~\ref{eq:ldef}, we have:
\begin{align}
g(\alphab,\betab)= \sum_{j\in\Vcal \backslash \Scal}\alpha_j -
\frac{1}{4\lambda|\Vcal \backslash \Scal|}\sum_{i\in \Vcal \backslash \Scal}\sum_{j\in \Vcal \backslash \Scal}\alpha_i \alpha_j y_i y_j\xb^\top _i \xb _j
\end{align}
where, since $\beta_j\ge 0$, according to Eq.~\ref{eq:linalpha} we have $\alpha_j\le 1$.
Then, the dual optimization problem would be
\begin{align}\label{eq:dualX}
   \underset{\alphab}{\text{maximize}}& \quad
\sum_{j\in\Vcal \backslash \Scal}\alpha_j -
\frac{1}{4\lambda|\Vcal \backslash \Scal|}\sum_{i\in \Vcal \backslash \Scal}\sum_{j\in \Vcal \backslash \Scal}\alpha_i \alpha_j y_i y_j\xb^\top _i \xb _j\\\
 \text{subject to} &\quad 0 \le \alpha_i   \le 1,\quad  i\in \Vcal \backslash \Scal \nonumber \\
  & \quad \sum_{i\in \Vcal \backslash \Scal}\alpha_i y_i  =0,  \nonumber
\end{align}
which is equivalent to
\begin{align*}
   \underset{\alphab}{\text{maximize}}& \quad
\alphab^\top \oneb_{|\Vcal|} - \frac{1}{4\lambda|\Vcal \backslash \Scal|}\sum_{i,j\in\Vcal}\alpha_i \alpha_j y_i y_j K(\xb _i, \xb _j)  
   \\
 \text{subject to} &\quad 0 \le \alpha_i   \le 1, \quad i \in\Vcal \\
 &\quad \alpha_i=0, \quad \qquad  i \in \Scal \\
 & \quad \sum_{i\in \Vcal \backslash \Scal}\alpha_i y_i  =0.  
\end{align*}
Moreover, if $\alphab^{*}(\Vcal \backslash \Scal)$ be the solution of the above optimization problem, then we get the required expression of $\wb^{*}(\Vcal \backslash \Scal)$ using Eq.~\ref{eq:duw}.

In case of (\emph{ii}), the proof is similar to \emph{(i)} except there is no presence of the differentiation in Eq.~\ref{eq:ddb}. 
The cases (\emph{iii}) and (\emph{iv}) follows similar proof technique. Finally, (v) follows from KKT conditions~\cite{cristianini2000introduction}. More in detail,
since $\alpha_i(y_i h_{\wbvs,\bbvs} (\xb_i) -1+\xi_i)=0$ for all $i\in\Vcal\cp\Scal$, we have:
\begin{align}
y_i h_{\wbvs,\bbvs} (\xb_i)>1 \implies \alpha_i=0; \nn
\end{align}
and, since $\xi_i \beta_i=0 \land \alpha_i +\beta_i=1 $ for all $i\in\Vcal\cp\Scal$, we have:
\begin{align}
& \xi_i>0 \implies \alpha_i=1 
\end{align}
which means,
\begin{align}
y_ih_{\wbvs,\bbvs} (\xb_i)<1 \implies \alpha_i=1
\end{align}

\end{proof}

%

 \begin{lemma}\label{lem:key}
 (I) For the soft margin nonlinear SVM with offset, if at least one training sample lies on the hyperplane,
 i.e., $\exists (\xb,y)\in \Vcal \backslash \Scal$ such that $y(\wb^{*}(\Vcal\cp\Scal)^\top \Phi(\xb)+ b^{*}(\Vcal \backslash \Scal))=1$, 
 then for any $\Lcal \subseteq \Vcal\cp\Scal$, we have that  $g(\Scal\cup\Lcal)-g(\Scal)$ is less than
\begin{align} 
    \dfrac{\left(2\sqrt{\lambda}+\kappa\right) |\Lcal|}{\sqrt{\lambda}}  +   \dfrac{\left(2\sqrt{\lambda}+\kappa\right)^2 |\Lcal|}{2\lambda}
   \left( \dfrac{1}{2} 
 + \sqrt{\dfrac{1}{4} +\dfrac{4\lambda |\Vcal|   b_{\max}}{(2\sqrt{\lambda}+\kappa)^2 |\Lcal|^2}} \right) 
  +|\Vcal| b_{\max}, \label{eq:fu1}
\end{align}
where $\kappa=\max_{\xb} \sqrt{K(\xb,\xb)}$ and $b_{\max} = 1+{\kappa}/{\sqrt{\lambda}}$.

(II) Given $\Lcal \subseteq \Vcal\cp\Scal$. Then, for the soft margin nonlinear SVM without offset, we have:
\begin{align}
 g(\Scal\cup \Lcal)& -g(\Scal) 
  \le
    \dfrac{\left(2\sqrt{\lambda}+\kappa\right) |\Lcal|}{\sqrt{\lambda}}  +   \dfrac{\left(2\sqrt{\lambda}+\kappa\right)^2 |\Lcal|}{2\lambda} ,\label{eq:fu2}
\end{align}
where $\kappa=\max_{\xb} \sqrt{K(\xb,\xb)}$.
\end{lemma}
\xhdr{Remark} Note that the assumption that one sample lies on the hyperplane is not restrictive. In general, there may be several optimal values of the offset
$b(\Vcal\cp\Scal)$. Therefore, even if a solution $(\wb^*(\Vcal\cp\Scal),b^{*}(\Vcal \backslash \Scal))$
does not satisfy such criteria (each point is strictly on one of the sides of the hyperplane), using Lemma~\ref{lem:bshift} we can shift $b^{*}(\Vcal \backslash \Scal) \leftarrow b^{*}(\Vcal \backslash \Scal) + b'$ to ensure this assumption, while incurring the same optimal loss.

\begin{proof}
We only prove ({I}), the result of (II) is readily given by setting  $b_{\max} = 0$ in Eq.~\ref{eq:fu1}. Our proof leverages the technique from~\cite{bousquet2002stability}. We first define
\begin{align*}
 &\df = \wb^*(\Vcal\cp (\Scal \cup \Lcal) ) -\wb^*(\Vcal\cp\Scal), \\
 &\dbb   = b^*(\Vcal\cp(\Scal\cup\Lcal)) -b^*(\Vcal\cp\Scal),\\
 &\macslack(h,y)=(1-y\cdot h)_+.
 \end{align*}
Next we define $\pred = (\wb,b)$ for any $\wb$ and $b$ and therefore, we can denote 
\begin{align*}
\pred^*(\Vcal\cp\Scal) & = (\wb^*(\Vcal\cp\Scal),b^*(\Vcal\cp\Scal))\\
 \dps 	& = {(\df,\dbb)}.
\end{align*}
Now, we define:
\begin{align*}
  R_{\Scal}(\pred)& =\sum_{i\in\Vcal \backslash \Scal} \macslack(h_{\wb,b}(\xb_i),y_i)
  = \sum_{i\in\Vcal \cp \Scal} \macslack( \wb ^\top \Phi(\xb_i) +b ,y_i).
\end{align*}
Since $\macslack$ is a convex function in $\pred$, so $R_{\Scal}$ is convex, and for $t \in [0,1] $, we have:
\begin{align*}
& R_{\Scal}\big(\pred^*(\Vcal\cp\Scal) + t \dps\big) - R_{\Scal} \big(\pred^*(\Vcal\cp\Scal)\big) 
\le  t \left( R_{\Scal}\big(\pred^*(\Vcal\cp(\Scal \cup \Lcal ))\big) -R_{\Scal}\big(\pred^*(\Vcal\cp\Scal)\big) \right)\\
& R_{\Scal}\big(\pred^*(\Vcal\cp(\Scal \cup \Lcal )) - t \dps\big) - R_{\Scal}\big(\pred^*(\Vcal\cp(\Scal \cup \Lcal ))\big)
\le  t\left( R_{\Scal}\big(\pred^*(\Vcal\cp\Scal)\big) -R_{\Scal}\big(\pred^*(\Vcal\cp(\Scal \cup \Lcal ))\big) \right) .
\end{align*}
Adding the above two equations, we have:
\begin{align}
  R_{\Scal}\big(\pred^*(\Vcal\cp\Scal) + t \dps\big)& + R_{\Scal}\big(\pred^*(\Vcal\cp(\Scal \cup \Lcal )) - t \dps\big)  
 \le R_{\Scal} \big(\pred^*(\Vcal\cp\Scal)\big)+ R_{\Scal}\big(\pred^*(\Vcal\cp(\Scal \cup \Lcal ))\big). \label{eq:RR}
\end{align}

Moreover since $\pred^*(\Vcal\cp\Scal) = (\wb^*(\Vcal\cp\Scal),b^*(\Vcal\cp\Scal))$ 
is the minimum of $\lambda |\Vcal \backslash \Scal|\cdot\bnm{\wb}^2  + R_{\Scal}((\wb,b))$,
we have:
\begin{align}
&\lambda |\Vcal \backslash \Scal|\bnm{\wb^*(\Vcal\cp\Scal)}^2  + R_{\Scal}\big(\pred^*(\Vcal\cp\Scal)\big) \nn\\
&\qquad\le \lambda |\Vcal \backslash \Scal|\cdot\bnm{ \wb^*(\Vcal\cp\Scal) + t \df }^2  + R_{\Scal}\big(\pred^*(\Vcal\cp\Scal) + t \dps\big), \label{eq:RRx1}
\end{align}
and similarly:
\begin{align}
\lambda& (|\Vcal \backslash \Scal|-|\Lcal|) \bnm{ \wb^*(\Vcal\cp (\Scal \cup \Lcal) ) }^2  + R_{\Scal \cup \Lcal}\big(\pred^*(\Vcal\cp(\Scal \cup \Lcal ))\big) \nn\\
& \le  \lambda |\Vcal \backslash \Scal|\cdot\bnm{ \wb^*(\Vcal\cp (\Scal \cup \Lcal) ) -t \df }^2 + R_{\Scal \cup \Lcal}\big(\pred^*(\Vcal\cp(\Scal \cup \Lcal ))-t\dps\big). \label{eq:RRx2}
\end{align}

Adding Eqs.~\ref{eq:RR},~\ref{eq:RRx1},~\ref{eq:RRx2},
we have:
\begin{align*}
& \lambda (|\Vcal \backslash \Scal|-|\Lcal|) \Big( \bnm{ \wb^*(\Vcal\cp (\Scal \cup \Lcal) ) }^2  - \bnm{ \wb^*(\Vcal\cp (\Scal \cup \Lcal) ) -t \df }^2\Big)\nn\\
&\qquad +\lambda |\Vcal \backslash \Scal| \Big(\bnm{ \wb^*(\Vcal\cp\Scal) }^2  -\bnm{ \wb^*(\Vcal\cp\Scal) + t \df}^2 \Big)\nn \\
&\qquad\qquad  \le \Big( R_{\Scal \cup \Lcal}\big(\pred^*(\Vcal\cp(\Scal \cup \Lcal ))-t\dps\big) - R_{\Scal}\big(\pred^*(\Vcal\cp(\Scal \cup \Lcal ))-t\dps\big) \Big)\nn\\ 
&\qquad\qquad\qquad-  \Big(  R_{\Scal \cup \Lcal}\big(\pred^*(\Vcal\cp(\Scal \cup \Lcal ))  \big) - R_{\Scal}\big(\pred^*(\Vcal\cp(\Scal \cup \Lcal ))\big) \Big).
\end{align*}
Therefore,
\begin{align}
& \lambda (|\vcs|-|\Lcal|) \bigg( \bnm{ \wbvsl }^2  - \bnm{ \wbvsl -t \df }^2\bigg)\nn\\
&\ +\lambda |\vcs| \bigg(\bnm{ \wbvs }^2  -\bnm{ \wbvs + t \df}^2 \bigg)\nn \\
&\qquad \quad   \le \bigg( R_{\Scal \cup \Lcal}(\prvsl-t\dps) - R_{\Scal}(\prvsl-t\dps) \bigg) \nn\\
&\qquad \quad  \quad-  \bigg(  R_{\Scal \cup \Lcal}(\prvsl  ) - R_{\Scal}(\prvsl) \bigg) \nn\\
%
\implies&  \bigg(2t\lambda (|\vcs|-|\Lcal|) {\inp{\df}{\wbvsl}} -t^2 \lambda (|\vcs|-|\Lcal|)\bnm{\df}^2 \bigg) \nn\\
&\ +\bigg(-2t\lambda |\vcs| {\inp{\df}{\wbvs}} - \lambda t^2 |\vcs|\cdot\bnm{\df}^2 \bigg)\nn\\
&\qquad \quad \le -R_{\Lcal} \left( \prvsl -t\dps \right) + R_{\Lcal} \left(\prvsl \right) \nn\\
&\qquad\quad = \sum_{i\in \Lcal} \bigg(\macslack(h_{\prvsl}(\xb_i),y_i)-\macslack\left(h_{\prvsl}(\xb_i)-t\cdot h_{\dps}(\xb_i),y_i\right)\bigg)\nn\\
\end{align}
\begin{align}
\overset{(i)}{\implies} &  2t\lambda |\vcs|\cdot\bnm{\df}^2 - 2t\lambda |\Lcal|\cdot\inp{\df} {\wbvsl}- t^2 \lambda (2|\vcs|-|\Lcal|) \bnm{\df}^2\nn \\
&\qquad \quad  \le  \sum_{i\in \Lcal} \bigg(\macslack(h_{\prvsl}(\xb_i),y_i)-\macslack\left(h_{\prvsl}(\xb_i)-t\cdot h_{\dps}(\xb_i),y_i\right)\bigg) \nn\\
&\qquad \quad \le t |\Lcal| |h_{\dps}(\xb_i)| \nn\\
\overset{(ii)}{\implies} &  2t\lambda |\vcs|\cdot\bnm{\df}^2 - 2t\lambda |\Lcal|\cdot\inp{\df} {\wbvsl}- t^2 \lambda (2|\vcs|-|\Lcal|) \bnm{\df}^2 \nn\\
&\qquad\quad  \le t |\Lcal| (|{\df}^\top \Phi(\xb_i)| +2b_{\max}) \nn\\
%
\overset{(iii)}{\implies} &  2t\lambda |\vcs|\cdot\bnm{\df}^2 - 2t\lambda |\Lcal|\cdot\inp{\df} {\wbvsl}- t^2 \lambda (2|\vcs|-|\Lcal|) \bnm{\df}^2 
\nn\\
&\qquad\quad   \le t |\Lcal| (\kappa\bnm{\df} +2b_{\max}) \nn\\
\end{align}
\begin{align}
\overset{(iv)}{\implies} &  2\lambda |\vcs|\cdot\bnm{\df}^2 - 2\lambda |\Lcal|\cdot\inp{\df} {\wbvsl}- t \lambda (2|\vcs|-|\Lcal|) \bnm{\df}^2 
\nn\\
&\qquad\quad  \le  |\Lcal| (\kappa\bnm{\df} +2b_{\max}) \nn \\
%
\overset{(v)}{\implies}& 2\lambda |\vcs|\cdot\bnm{\df}^2 - 2\sqrt{\lambda} |\Lcal|\cdot\bnm{\df} - t \lambda (2|\vcs|-|\Lcal|) \bnm{\df}^2 
\nn\\
&\qquad\quad   \le  |\Lcal| (\kappa\bnm{\df} +2b_{\max})\label{eq:lastl1},
\end{align}
where 
\begin{itemize}
 \item  $(i)$ is due to that:
\begin{align}
  2t&\lambda (|\vcs|-|\Lcal|) {\inp{\df}{\wbvsl}}  - 2t\lambda |\vcs| {\inp{\df}{\wbvs}}\nn\\
  &= 2 t \lambda |\vcs|\cdot\inp{\df}{ (\wbvsl -\wbvs) }  -2t\lambda|\Lcal|\cdot\inp{\df} {\wbvsl}\nn\\
  &= 2 t \lambda |\vcs|\cdot\bnm{\df}^2   -2t\lambda|\Lcal|\cdot\inp{\df} {\wbvsl};\nn
 \end{align}
$\macslack(h_{\prvsl}(\xb_i),y_i)-\macslack\left(h_{\prvsl}(\xb_i)-t\cdot h_{\dps}(\xb_i),y_i\right) 
 \le      t|h_{\dps}(\xb_i)| \left|\frac{\dd \macslack(y,p)}{\dd p} \right|_{\max}$
with $ \text{sup}_p \left |\frac{\dd \macslack(y,p)}{\dd p} \right| =1$ for hinge loss; 
\item  ${(ii)}$ is due to that $|h_{\dps}(\xb)|=|{\df}^\top \Phi(\xb)+\dbb|\le |{\df}^\top \Phi(\xb)|+2b_{\max}$, with $b_{\max}=1+\frac{\kappa}{\lambda}$ (Fact~\ref{fact:bbound}); 
\item $(iii)$ is due to that $|{\df}^\top \Phi(\xb)|\le \bnm{\df} \sqrt{K(x,x)}\le \kappa \bnm{\df}$ by Cauchy-Schwartz inequality;
\item $(iv)$ is due to canceling $t \ge 0$ from both sides of the inequality; and,
\item ${(v)}$ is due to that $\inp{\df} {\wbvsl} \le \bnm{\wbvsl}  \bnm{\df}$ by Cauchy-Schwartz inequality
and $\bnm{\wbvsl}^2 \le \frac{1}{\lambda}$ by Fact~\ref{fact:fupper}.
 \end{itemize}
%
%
%
%
Eq.~\ref{eq:lastl1} is valid for all $t\in[0,1]$. Making $t\to 0$, we have :
\begin{align}
 &2\lambda |\vcs|\cdot\bnm{\df}^2 - (2\sqrt{\lambda} |\Lcal|+\kappa |\Lcal|) \bnm{\df} \le 2b_{\max}\nn \\
 \implies & \bnm{\df}  \le \dfrac{1}{ 2\lambda |\vcs|} \left( {\dfrac{(2\sqrt{\lambda}+\kappa)|\Lcal|}{2} 
 + \sqrt{\dfrac{(2\sqrt{\lambda}+\kappa)^2 |\Lcal|^2 }{4} + 4  \lambda |\vcs|   b_{\max}}} \right).\label{eq:dw-expression}
\end{align}
We slightly modify here the first part in the proof of {Proposition}~\ref{thm:mont} to note that
\begin{align}
 g(\Scal  & \cup \Lcal)  - g(\Scal)\nn\\
 &= \sum_{i\in\Vcal \backslash \Scal} [\lambda \bnm{\wb^{*}(\Vcal \backslash \Scal)}^2 + \macslack(h_{\prvs}(\xb_i),y_i) ]\nn \\
&\quad - \sum_{i\in{\Vcal \backslash (\Scal\cup \Lcal)}} [ \lambda \bnm{\wb^{*}(\Vcal \backslash (\Scal\cup \Lcal))}^2 + \macslack(h_{\prvsl}(\xb_i),y_i) ] \nn \\
 & =  \lambda |\Vcal \backslash \Scal| \bnm{\wb^{*}(\Vcal \backslash \Scal)}^2  + \sum_{i\in\Vcal \backslash \Scal} \macslack(h_{\prvs}(\xb_i),y_i)\nn \\
  &\quad -  \lambda (|\Vcal \backslash \Scal|-|\Lcal|) \bnm{\wb^{*}(\Vcal \backslash \Scal)}^2 - \sum_{i\in \Vcal \backslash (\Scal\cup \Lcal)}
  \macslack(h_{\prvs}(\xb_i),y_i)\nn \\
 &\quad +  \lambda (|\Vcal \backslash \Scal|-|\Lcal|) \bnm{\wb^{*}(\Vcal \backslash \Scal)}^2  + \sum_{i\in \Vcal \backslash (\Scal\cup \Lcal) } \macslack(h_{\prvs}(\xb_i),y_i)\nn \\
&\quad - \lambda (|\Vcal \backslash \Scal|-|\Lcal|) \bnm{\wb^{*}(\Vcal \backslash (\Scal \cup \Lcal))}^2 - \sum_{i\in{\Vcal \backslash (\Scal\cup \Lcal)}} \macslack(h_{\prvsl}(\xb_i),y_i)\nn \\
&= \lambda |\Lcal| \cdot \bnm{\wb^{*}(\Vcal \backslash \Scal)}^2  +  \sum_{j \in\Lcal} \macslack(h_{\prvs}(\xb_j),y_j) \nn\\
&\quad +  \lambda (|\Vcal \backslash \Scal|-|\Lcal|) \bnm{\wb^{*}(\Vcal \backslash \Scal)}^2  + \sum_{i\in \Vcal \backslash (\Scal\cup \Lcal) } \macslack(h_{\prvs}(\xb_i),y_i) \nn\\
 &\quad
  - \lambda (|\Vcal \backslash \Scal|-|\Lcal|) \bnm{\wb^{*}(\Vcal \backslash (\Scal \cup \Lcal))}^2 - \sum_{i\in{\Vcal \backslash (\Scal\cup \Lcal)}} \macslack(h_{\prvsl}(\xb_i),y_i). \label{eq:imm}
%
\end{align}
%
%
%
Now we bound the first part of Eq.~\ref{eq:imm}
\begin{align}
 \lambda  |\Lcal|\cdot\bnm{\wbvs}^2  +  \sum_{j \in\Lcal} \macslack(h_{\prvs}(\xb_j),y_j) 
& \overset{(i)}{\le} |\Lcal| +|\Lcal| ( 1 + \kappa \bnm{\wbvs}  +b_{\max})\nn\\
& \overset{(ii)}{\le} |\Lcal| +|\Lcal| \left( 1 + \kappa \frac{1}{\sqrt{\lambda}} +b_{\max}\right)  \nn\\
&= |\Lcal| \left( 2 + \kappa \frac{1}{\sqrt{\lambda}} +b_{\max}\right),  \label{eq:f1}
\end{align}
where
\begin{itemize}
\item inequality $(i)$ is due to $\bnm{\wbvs}^2  \le \frac{1}{\lambda}$ and $ \macslack(h_{\prvs}(\xb_j) ,  y_j) \le 1 + |\wbvs^\top \Phi(\xb_j)|
+b_{\max} \le 1 + \kappa \bnm{\wbvs} + b_{\max}$\nn\\
\item inequality ${(ii)}$ is due to  $\bnm{\wbvs}^2  \le \frac{1}{\lambda}$.
\end{itemize}
%
%
Next, we bound the second part of Eq.~\ref{eq:imm}  in the following.
\begin{align}
 & \left(\lambda (|\vcs|-|\Lcal|) \bnm{\wbvs}^2  + \sum_{i\in \Vcal\cp(\Scal\cup\Lcal) } \macslack(h_{\prvs}(\xb_i), y_i) \right)  \nn\\
 & \ -  \left( \lambda (|\vcs|-|\Lcal|) \bnm{\wbvsl}^2 + \sum_{i\in{\Vcal\cp(\Scal\cup\Lcal)}} \macslack(h_{\prvsl}(\xb_i), y_i)\right)\nn
    \end{align}
  \begin{align}
  &\ \le \lambda(|\vcs|-|\Lcal|) \inp{(\wbvsl -\wbvs)}{(\wbvsl + \wbvs)} \nn\\
 & \quad +  \sum_{i\in \Vcal\cp(\Scal\cup\Lcal) } \left( \macslack(h_{\prvs}(\xb_i), y_i) ) - \macslack(h_{\prvsl}(\xb_i), y_i) \right)  \nn\\
   &\ \overset{(i)}{\le}  \lambda(|\vcs|-|\Lcal|) \bnm{\df}  \bnm{\wbvsl + \wbvs} 
   + \sum_{i\in \Vcal\cp(\Scal\cup\Lcal)}\text{sup}_p \left| \frac{\dd \macslack(p,y)}{\dd p}\right| \cdot |h_{\dps}(\xb_i)| \nn\\
  &\  \overset{(ii)}{\le}   2\sqrt{\lambda}(|\vcs|-|\Lcal|) \bnm{\df} 
   + (|\vcs|-|\Lcal|)  \cdot (\kappa \bnm{\df} +b_{\max})\nn\\
 &  =(|\vcs|-|\Lcal|) (2\sqrt{\lambda} + \kappa) \bnm{\df}  + (|\vcs|-|\Lcal|) b_{\max},\label{eq:f2}
   \end{align}
where, 
\begin{itemize}
\item inequality ${(i)}$ is because:
\begin{align*}
& \inp{(\wbvsl -\wbvs)}{(\wbvsl + \wbvs)} \nn\\
 &\qquad\qquad \le \bnm{\df}  \bnm{\wbvsl + \wbvs},
\end{align*}
 by Cauchy-Schwartz inequality and,
\begin{align*}
\left( \macslack(h_{\prvs}(\xb_i), y_i) ) - \macslack(h_{\prvsl}(\xb_i), y_i) \right) \
\le \text{sup}_p \left| \frac{\dd \macslack(p,y)}{\dd p}\right|  \cdot |h_{\dps}(\xb_i)|  ,
\end{align*}
by Lipschitz criterion;
\item inequality ${(ii)}$ is due to that $\bnm{\wbvsl + \wbvs}  \le  \bnm{\wbvsl} + \bnm{\wbvs}  \le \frac{2}{\sqrt{\lambda}}$ from
Fact~\ref{fact:fupper}, the fact that $\text{sup}_p \left| \frac{\dd \macslack(p,y)}{\dd p}\right|=1$ and $|h_{\dps}(\xb)| \le |{\df}^\top \Phi(\xb)|+b_{\max}\le 
\kappa\bnm{\df} + b_{\max}$ from triangle inequality and Cauchy Schwartz inequality, respectively. 
\end{itemize}
By adding Eqs.~\ref{eq:f1} and~\ref{eq:f2} and then replacing $\bnm{\df}$ with its value provided by Eq.~\ref{eq:dw-expression}
we have that $g(\Scal \cup \Lcal) -  g(\Scal )$ is smaller than
\begin{align}
 \hspace{-4mm} & \left( 2 + \kappa \frac{1}{\sqrt{\lambda}}\right)|\Lcal|\nn\\ 
&\ +   \dfrac{(|\vcs| -|\Lcal|)}{ 2\lambda |\vcs|} \cdot  (2\sqrt{\lambda}+\kappa)\cdot \left( {\dfrac{(2\sqrt{\lambda}+\kappa)|\Lcal|}{2} 
 + \sqrt{\dfrac{(2\sqrt{\lambda}+\kappa)^2 |\Lcal|^2 }{4} + 4  \lambda |\vcs|   b_{\max}}} \right)  +|\vcs| b_{\max}\nn\\
 &\quad\le  
    \dfrac{\left(2\sqrt{\lambda}+\kappa\right) |\Lcal|}{\sqrt{\lambda}}  +   \dfrac{(2\sqrt{\lambda}+\kappa)^2 |\Lcal|}{2\lambda}
  \cdot \left( \dfrac{1}{2} 
 + \sqrt{\dfrac{1}{4} +\dfrac{4\lambda |\vcs|   b_{\max}}{\left(2\sqrt{\lambda}+\kappa\right)^2 |\Lcal|^2}} \right) +|\vcs| b_{\max} \nn\\
  &\quad\le  
    \dfrac{\left(2\sqrt{\lambda}+\kappa\right) |\Lcal|}{\sqrt{\lambda}}  +   \dfrac{\left(2\sqrt{\lambda}+\kappa\right)^2 |\Lcal|}{2\lambda}
  \cdot \left( \dfrac{1}{2} 
 + \sqrt{\dfrac{1}{4} +\dfrac{4\lambda |\Vcal|   b_{\max}}{(2\sqrt{\lambda}+\kappa)^2 |\Lcal|^2}} \right) +|\Vcal| b_{\max} \nn
 \end{align}

  \end{proof}

\vspace{-2mm}
\begin{lemma}\label{lem:bshift}
For soft margin SVMs with offset, if no training instance lies on the hyperplane, \ie, $\nexists (\xb,y)$ so that $y(\wb^*(\Vcal\cp\Scal)^\top \Phi(\xb)+b^{*}(\Vcal \backslash \Scal))=1$, then one can compute a shift value $b'$ so that  $(\wb^*(\Vcal\cp\Scal),b^*(\Vcal\cp\Scal)+b')$ is also an optimal solution and it ensures that at least one instance $(\xb,y)$ satisfies $y(\wb^*(\Vcal\cp\Scal)^{\top} \Phi(\xb)+b^{*}(\Vcal \backslash \Scal)+b')=1$.
\end{lemma}
\begin{proof}
 Since there is no training instance that satisfies  $y(\wb^*(\Vcal\cp\Scal)^{\top} \Phi(\xb)+b^{*}(\Vcal \backslash \Scal))=1$, then for each instance either $y(\wb^*(\Vcal\cp\Scal)^{\top} \Phi(\xb)+b^{*}(\Vcal \backslash \Scal)) > 1$ or $y(\wb^*(\Vcal\cp\Scal)^{\top} \Phi(\xb)+b^{*}(\Vcal \backslash \Scal))<1$. First, we assume  there exists at least one $(\xb_i,y_i)$, such that $y_i (\wb^*(\Vcal\cp\Scal)^{\top} \Phi(\xb_i)+b^{*}(\Vcal \backslash \Scal)) > 1$. Let us define,
\begin{align}
 \xi& = \min_{i\in\Vcal \backslash \Scal}\left\{ \left|1-y_i (\wb^*(\Vcal\cp\Scal)^{\top} \Phi(\xb_i) +b^{*}(\Vcal \backslash \Scal))\right| \right\}\nonumber \\
  i^* & \in \argmin_{i\in\Vcal \backslash \Scal}\left\{ \left|1-y_i (\wb^*(\Vcal\cp\Scal)^{\top} \Phi(\xb_i) +b^{*}(\Vcal \backslash \Scal))\right| \right\}. \nonumber
\end{align}
If we shift $b^{*}(\Vcal \backslash \Scal) \to b^{*}(\Vcal \backslash \Scal) - y_{i^*} \xi $, then we have:
\begin{align}
 1-y_i (\wb^*(\Vcal\cp\Scal)^{\top} \Phi(\xb_i) +b^{*}(\Vcal \backslash \Scal)) < 0 &\implies 1-y_i (\wb^*(\Vcal\cp\Scal)^{\top} \Phi(\xb_i) +b^{*}(\Vcal \backslash \Scal)-y_{i^*} \xi)\le 0. \label{eq:bbx1} \\
  1-y_i (\wb^*(\Vcal\cp\Scal)^{\top} \Phi(\xb_i) +b^{*}(\Vcal \backslash \Scal)) > 0 & \implies 1-y_i (\wb^*(\Vcal\cp\Scal)^{\top} \Phi(\xb_i) +b^{*}(\Vcal \backslash \Scal)-y_{i^*} \xi)\ge 0. \label{eq:bbx2}
\end{align}
Now, we define
\begin{align}
\Mcal_{\cmark}: = \{i\in \Vcal \backslash \Scal \,|\,1-y_i (\wb^*(\Vcal\cp\Scal)^{\top} \Phi(\xb_i) +b^{*}(\Vcal \backslash \Scal)) <0 \} \\
\Mcal_{\xmark}: = \{i\in\Vcal \backslash \Scal\,|\,1-y_i (\wb^*(\Vcal\cp\Scal)^{\top} \Phi(\xb_i) +b^{*}(\Vcal \backslash \Scal))> 0\}
\end{align}
Hence, we have:
\begin{align}
 \sum_{i \in \Vcal \backslash \Scal} &[1-y_i (\wb^*(\Vcal\cp\Scal)^{\top} \Phi(\xb_i) +b^{*}(\Vcal \backslash \Scal)-y_{i^*} \xi)]_+ \nonumber \\
 &= \sum_{i \in \Mcal_{\cmark}} [1-y_i (\wb^*(\Vcal\cp\Scal)^{\top} \Phi(\xb_i) +b^{*}(\Vcal \backslash \Scal)-y_{i^*} \xi)]_+ \nn\\
&\quad +\sum_{i \in \Mcal_{\xmark}} [1-y_i (\wb^*(\Vcal\cp\Scal)^{\top} \Phi(\xb_i) +b^{*}(\Vcal \backslash \Scal)-y_{i^*} \xi)]_+\nonumber
\end{align}
\begin{align}
& \overset{(i)}{=} \sum_{i \in \Mcal_{\xmark}} (1-y_i (\wb^*(\Vcal\cp\Scal)^{\top} \Phi(\xb_i) +b^{*}(\Vcal \backslash \Scal)-y_{i^*} \xi)) \nonumber\\
& = \sum_{i \in \Mcal_{\xmark}} (1-y_i (\wb^*(\Vcal\cp\Scal)^{\top} \Phi(\xb_i) +b^{*}(\Vcal \backslash \Scal))  -  [y_{i^*} \xi]\cdot \bigg(\sum_{i\in\Mcal_{\xmark}} y_i\bigg)    \nonumber\\
& \overset{(ii)}{=} \sum_{i \in \Mcal_{\xmark}} (1-y_i (\wb^*(\Vcal\cp\Scal)^{\top} \Phi(\xb_i) +b^{*}(\Vcal \backslash \Scal)), \label{eq:fll}
\end{align}
where $(i)$ is due to Eqs.~\ref{eq:bbx1} and~\ref{eq:bbx2}. To show $(ii)$, we know from Lemma~\ref{prop:svm-dual} that $\sum_{i\in\Vcal \backslash \Scal}\alpha_{i}(\Vcal \backslash \Scal) y_i=0$; $\alpha_{i}(\Vcal \backslash \Scal)=0$ 
for  $i\in\Mcal_{\cmark}$ ;
and $\alpha_{i}(\Vcal \backslash \Scal)=1$ 
for $i\in \Mcal_{\xmark}$.
Since there is no $(\xb,y)$
such that $y(\wb^*(\Vcal\cp\Scal)^{\top} \Phi(\xb) + b^{*}(\Vcal \backslash \Scal))=1$, 
then we have  $\sum_{i\in\Vcal \backslash \Scal_{\xmark}} \alpha_i y_i = \sum_{i\in\Vcal \backslash \Scal_{\xmark}} y_i =0$.

From Eq.~\ref{eq:fll} we can see that the loss function remains unchanged, Therefore $(\wb^*(\Vcal\cp\Scal),b^{*}(\Vcal \backslash \Scal) +b' )$  is also an optimal solution, where $b' = -y_{i^*} \xi$. Moreover, we have:
\begin{align}
1- y_{i^*} (\wb^*(\Vcal\cp\Scal)^\top \Phi(\xb_{i^*}) +b^{*}(\Vcal \backslash \Scal)-y_{i^*}\xi) = 1- y_{i^*} (\wb^*(\Vcal\cp\Scal)\Phi(\xb_{i^*}) +b^{*}(\Vcal \backslash \Scal))-\xi = 0,\nn
\end{align}
which completes the proof.
\end{proof}
%
\begin{claim} \label{fact:fupper}
For kernel SVMs without offset, we have $\forall\ \Scal\subseteq \Vcal: \bnm{\wb^*(\Vcal\cp\Scal)}^2  \le {1}/{\lambda}$ .
\end{claim}
\begin{proof}
We have $\lambda |\Vcal \backslash \Scal| \bnm{\wb^*(\Vcal\cp\Scal)}^2  \le  \lambda |\Vcal \backslash \Scal| \bnm{\wb^*(\Vcal\cp\Scal)}^2  + \sum_{i\in\Vcal \backslash \Scal} \mac(\pred(\xb_i),y_i)  
 \le \lambda |\Vcal \backslash \Scal| .0 + \sum_{i\in\Vcal \backslash \Scal} \mac(0,y_i)=|\Vcal \backslash \Scal| $,
which proves the required result.
\end{proof}
 
\begin{claim}\label{fact:bbound}
For a  nonlinear  soft margin SVM with offsets, if there exists at least one training instance, which satisfies $y(\wb^*(\Vcal\cp\Scal)^{\top}\Phi( \xb)+b^{*}(\Vcal \backslash \Scal))=1$, then we have $|b^{*}(\Vcal \backslash \Scal)|\le 1 + \kappa/\sqrt{\lambda}$, where $\kappa=\max_{\xb} \sqrt{K(\xb,\xb)} \ge 0$
\end{claim}
\begin{proof}
Since $ y(\wb^*(\Vcal\cp\Scal)^{\top} \Phi(\xb)+b^{*}(\Vcal \backslash \Scal))=1$, then $b^{*}(\Vcal \backslash \Scal) = y- \wb^*(\Vcal\cp\Scal)^{\top} \Phi(\xb)$, and we have:
\begin{align}
 |b^{*}(\Vcal \backslash \Scal)| &= |y- \wb^*(\Vcal\cp\Scal)^{\top} \Phi(\xb)|\nonumber\\
 & \le 1 + |\wb^*(\Vcal\cp\Scal)^{\top} \Phi(\xb)|\nonumber\\
 & \le 1+ \bnm{f}_\HH \sqrt{K(\xb,\xb)} \qquad (\text{Fact}~\ref{fact:csi})\nonumber\\
 & \le 1+ \kappa/\sqrt{\lambda} \qquad (\text{Fact}~\ref{fact:fupper}).
\end{align}
\end{proof}
\begin{claim}[Cauchy-Schwartz Inequality] \label{fact:csi}
$|f(\xb)|\le \bnm{f} \sqrt{K(\xb,\xb)}$.
\end{claim}

\begin{claim}\label{fact:scalBoundst}
If  $|\Scal|<n < \min\{|\Vcal^+|,|\Vcal^-|\}-s^*$ for some $s^*\in \nonumber^+$, then $|(\Vcal \backslash \Scal)\cap\Vcal^{\pm}| > s^*$.
\end{claim}
\begin{proof}
We have:
 \begin{align}
  |(\Vcal \backslash \Scal)\cap\Vcal^-|& = |\Vcal^-\cp \Scal|\nonumber\\
  & =  |\Vcal^-|- |\Scal\cap\Vcal^-|\nonumber\\
  & \ge |\Vcal^-|- |\Scal|\nonumber\\
  & \ge |\Vcal^-|- \min\{\|\Vcal^+|,|\Vcal^-|\}+s^* > s^*.
 \end{align}
\end{proof}

\begin{claim}\label{fact:ccalprop1}
 Given $\Ccal^\pm _{1/s}$, $\Delta_{1/s}$ are defined according to Eqs.~\ref{eq:cnp} and~\ref{eq:deltas},
then $\Delta_{1/s_1} \ge \Delta_{1/s_2}$ if $s_1 > s_2$.
\end{claim}
\begin{proof}
Since $1/s_1 < 1/s_2$, then $\Ccal^\pm _{1/s_1} \subseteq \Ccal^\pm _{1/s_2} $. Therefore,
 any two points in $\Ccal^+ _{1/s_1}$ and $\Ccal^- _{1/s_1}$, also are in 
 $\Ccal^+ _{1/s_2}$ and $\Ccal^- _{1/s_2}$,  respectively. As a result, the minimum distance between two points in
 $\Ccal^+ _{1/s_2}$ and $\Ccal^- _{1/s_2}$ would be smaller.
\end{proof}

\begin{claim}\label{fact:quadratic}
 If $g(\alpha)=\alpha-p \alpha^2$, then $\max_{\alpha\in [0,1]} g(\alpha) \ge \min\left\{   \dfrac{1}{4p}, \dfrac{1}{2}  \right\}$.
\end{claim}
\begin{proof}
Assume $\alpha^*$ indicates the global maximum.  Now, $\left.\dfrac{dg(\alpha)}{d\alpha} \right|_{\alpha=\alpha^*} =0$
iff $\alpha^* <1$. Otherwise $\alpha^*=1$. Hence, if $1/2p <1$, then  $\max_{\alpha\in [0,1]} g(\alpha) = 1/4p$.
Otherwise, we have $\max_{\alpha\in [0,1]} g(\alpha) = 1-p \ge 1/2$. 
\end{proof}

\section{Additional Details for Experiments on Real Data} \label{app:impl-real}
\xhdr{Dirichlet parameters $\chib_{\bullet,\bullet}$} Given an instance $(\xb,y)$ with grade $q$, the values of $\chi_{\xb,q}$ depend only on $q$, which vary across different datasets. More specifically, we set $\xi_{\xb,q}$ for our datasets as follows:
\begin{itemize}
\item Messidor. It contains scores on four point scale. Hence, we have:
\begin{align}
\chib_{\xb,q}= \begin{cases}
[3,\, 3,\,1,\, 1] \quad \text{if }\ q=1\\
[2,\, 3,\,2,\, 1] \quad \text{if }\ q=2\\
 [0.5,\, 0.5,\,5,\, 4] \quad \text{if }\ q=3\\
[0.1,\, 0.1,\,4,\, 6] \quad \text{if }\ q=4
 \end{cases} 
\end{align}
\item Stare: It contains scores on five point scale. Hence, we have:
\begin{align}
\chib_{\xb,q}= \begin{cases}
[3,\, 3,\,2,\, 1,\, 1]\quad \text{if }\ q=1\\
[2,\, 7,\,0.5,\, 0.5,\, 0.1]\quad \text{if }\ q=2\\
[0.1,\, 0.1,\,4,\, 3,\, 2] \quad \text{if }\ q=3\\
[1,\, 2,\,3,\, 3,\, 1]\quad \text{if }\ q=4\\
[0.1,\, 0.1,\,5,\, 5,\, 5] \quad \text{if }\ q=5
 \end{cases} 
\end{align}
\item Aptos: It contains scores on five point scale. Hence, we have:
\begin{align}
\chib_{\xb,q}= \begin{cases}
[4,\, 2,\,1,\, 1,\, 1]\quad \text{if }\ q=1\\
[4,\, 1,\,1,\, 0.5,\, 0.5]\quad \text{if }\ q=2\\
[0.1,\, 0.1,\,5,\, 4,\, 4] \quad \text{if }\ q=3\\
[0.1,\, 0.1,\,4,\, 5,\, 4]\quad \text{if }\ q=4\\
[0.1,\, 0.1,\,4,\, 4,\, 5] \quad \text{if }\ q=5
 \end{cases} 
\end{align}
\end{itemize}

\xhdr{Choice of the additional model $\pi$} For all three real datasets, we use a logistic regression model for $\pi$, which is given as follows:
\begin{align}
 \pi(d \, |\,\xb) = \dfrac{1}{1+\exp\left[-\omega_{\pi} \cdot d   (\wb^*(\Vcal\cp\Scal^*)^\top  \xb  + b^*(\Vcal\cp\Scal^*) )\right]}
\end{align}
where $\omega_{\pi}$ is the trainable paramater.

\xhdr{Implementation of Triage based on predicted errors}
We first train a support vector machine for full automation, i.e., $\Scal=\emptyset$
as well as two additional supervised models that predict the human error and the machine 
error. Finally, we sort the test samples in decreasing order of the difference between the predicted machine error and the predicted human error and outsource top $n$ samples to humans. 

Following~\cite{raghu2019algorithmic}, we train two multi-layer perceptron models--- one ($\Ecal_h$) for predicting human error and another ($\Ecal_m$) for predicting machine error--- 
on $\{(\xb_i,z_i)\}_{i\in\Vcal}$ where $\xb_\bullet$ are the feature and $z_\bullet$ are the binary labels. 
The binary labels $z_\bullet$
are different for $\Ecal_h$ and $\Ecal_m$ and are computed as follows. 

For the human error model $\Ecal_h$, we first sample  
four expert scores per each sample $(\xb,y)$ using the same strategy described in Section~\ref{sec:real},
then split them into two evenly sized sets and, finally aggregate the grades in each set into single binary labels by averaging
and thresholding using the same thresholding strategy described in Section~\ref{sec:real}. 
We set $z=0$ if these binary labels for two sets agree and $z=1$ otherwise.
On the other hand, for machine error model $\Ecal_m$, we set $z=0$, 
if, for the sample $(\xb,y)$, the label predicted by the trained machine agrees with the label predicted by human and $z=1$ otherwise.

Each of  $\Ecal_h$ and $\Ecal_m$ consists of one hidden layer with 100 units with  ReLU$(\cdot)$ as the activation function for the hidden layer.
Moreover, we use $L_2$ regularization and optimize using stochastic gradient descent.

\xhdr{Computing infrastructure} All experiments were implemented in Python $2.7$ and were carried out in  $2.2$ GHz Intel Core i7 processor with $8$ GB of RAM.

\xhdr{Average run-time of each result}
%
The average run-time of each iteration of distorted greedy algorithm is $6.91, 1.26$ and $10.79$ seconds for Messidor, Stare and Aptos datasets respectively. 
On the other hand, the average run-time of each iteration of stochastic 
distorted greedy algorithm has decreased to  $0.98, 0.3$ and $0.97$ seconds for Messidor, Stare and Aptos datasets respectively. 
\end{appendix}

\end{document}